\documentclass[journal]{IEEEtai}
\IEEEoverridecommandlockouts

\usepackage{cite}
\usepackage{amsmath,amssymb,amsfonts}
\usepackage{amsthm}
\newtheorem{proposition}{Proposition}
\newtheorem{theorem}{Theorem}
\newtheorem{lemma}{Lemma}
\usepackage{algorithm}
\usepackage{algorithmic}
\usepackage{graphicx}
\usepackage{textcomp}
\usepackage{xcolor}
\usepackage{booktabs}
\usepackage{array}
\usepackage{url}
\usepackage{multirow}
\usepackage{arydshln}
\usepackage{soul}
\usepackage{hyperref}
\hypersetup{
    colorlinks = true,
    linkcolor  = blue,
    citecolor  = blue,
    urlcolor   = blue,
    pdfencoding = auto
}
\usepackage{silence}
\newcommand{\indicator}{\mathbf{1}}

\def\BibTeX{{\rm B\kern-.05em{\sc i\kern-.025em b}\kern-.08em
    T\kern-.1667em\lower.7ex\hbox{E}\kern-.125emX}}

\begin{document}

\title{FedRAW: Preserving Rare-Label Influence in Asynchronous Federated Learning}

\author{

\IEEEauthorblockN{
Prashant Bajpai\IEEEauthorrefmark{1},
Divya Saxena\IEEEauthorrefmark{2},
Philippe Lalanda\IEEEauthorrefmark{3},
German Vega\IEEEauthorrefmark{3}
}

\IEEEauthorblockA{
\IEEEauthorrefmark{1}School of Artificial Intelligence and Data Science, Indian Institute of Technology Jodhpur, India\\
\IEEEauthorrefmark{2}Department of Computer Science and Engineering, Indian Institute of Technology Jodhpur, India\\
\IEEEauthorrefmark{3}Laboratoire d'Informatique de Grenoble (LIG), University of Grenoble Alpes, Grenoble, France\\
}

\thanks{\IEEEauthorrefmark{2}Corresponding author: divyasaxena@iitj.ac.in}

}

\maketitle{}

\begin{abstract}
Asynchronous federated learning improves scalability by updating the global model from a server-side buffer of client updates as they arrive, rather than waiting for all selected clients to finish. While efficient, this arrival-driven aggregation can silently distort representation learning under heterogeneous participation. We identify silent rarity failure, a hidden failure mode in which clients holding rare labels contribute too weakly to the global model even though its overall accuracy appears largely unaffected. This failure arises from two coupled effects: rare-label clients may submit updates less frequently when they are slower or less available, creating participation bias; and once their updates enter the buffer, standard asynchronous aggregation assigns them no compensating influence, creating aggregation bias. We propose FedRAW, a fully server-side aggregation method that preserves rare-label influence without changing local training, client objectives, or communication protocols. FedRAW combines client-level update deduplication, which prevents frequently arriving clients from repeatedly dominating the update buffer, with rare-label-aware weighting, which increases the influence of clients carrying low-coverage labels. We formalize silent rarity failure through participation and aggregation bias, and show that FedRAW increases rare-label client influence over uniform aggregation while preserving convergence. Across EMNIST Balanced, CIFAR-10, HAM10000, and ISIC-2019, FedRAW improves rare-label accuracy while preserving comparable global accuracy and adding negligible server-side computation.
\end{abstract}

\begin{IEEEImpStatement}
Asynchronous federated learning improves scalability by updating the global model from a server-side buffer as client updates arrive, rather than waiting for all selected clients to finish. However, this arrival-driven aggregation can silently suppress the
influence of clients holding low-coverage labels, while leaving global
accuracy largely unaffected. FedRAW addresses this \textit{silent rarity failure} through server-side buffer deduplication and rare-label-aware aggregation, without modifying ongoing client training or communication protocols. The approach provides a lightweight mechanism for preserving rare-label influence in asynchronous aggregation and can improve the reliability of federated models in settings where some labels have low client coverage. More broadly, this work shows that global accuracy alone can mask rare-label failures in asynchronous federated learning.
\end{IEEEImpStatement}


\section{Introduction}

Asynchronous federated learning improves scalability by updating the global model from a server-side buffer of client updates as they arrive, rather than waiting for all selected clients to finish. By reducing the straggler bottleneck of synchronous training, arrival-driven aggregation better matches the realities of heterogeneous federated systems, in which clients differ in computation speed, communication bandwidth, and availability.

However, the same arrival-driven mechanism that improves scalability can also distort representation learning under heterogeneous participation. In practical federated systems, client speed, availability, data volume, and label coverage may be correlated: clients that are slower or less frequently available may also hold smaller, more specialized, or less frequently observed data distributions~\cite{pati2022federated}. We identify the resulting failure mode as \textbf{silent rarity failure}, where clients holding rare labels contribute too weakly to the global model even though overall accuracy appears largely unaffected. This failure arises from two coupled effects: rare-label clients may submit updates less frequently, creating \textbf{participation bias}; and once their updates enter the server-side buffer, standard asynchronous aggregation assigns them uniform or rarity-agnostic influence, creating \textbf{aggregation bias}. Critically, this is a \textit{system-induced} failure: it arises from the dynamics of asynchronous buffer aggregation rather than from the global sample distribution, and occurs even when rare-label clients hold sufficient local data to represent their classes well. Together, these effects suppress rare-label influence at the aggregation level and make rare-label degradation hard to detect from global accuracy alone.

We first expose silent rarity failure in a controlled EMNIST Balanced setting under FedFa~\cite{yi2024}, where the two biases can be isolated (Fig.~\ref{fig:motivation}): a control setting with neither rare-label concentration nor speed heterogeneity (\textbf{Exp 1}), a rare-label setting with uniform client speeds that isolates aggregation bias (\textbf{Exp 2}), and a correlated-speed setting where rare-label clients are also slower, combining both biases (\textbf{Exp 3}). Global accuracy stays nearly unchanged at 75--77\% across all three, yet rare-label accuracy drops to 16.9\% under aggregation bias alone and collapses to 2.3\% when both are present -- confirming that asynchronous FL can silently lose rare-label representations while global evaluation suggests normal convergence.

Existing FL methods address related but distinct challenges: staleness-aware asynchronous FL compensates for delayed or outdated updates, buffer-based methods improve throughput under uniform or rarity-agnostic aggregation, fairness-aware FL targets client-level or demographic parity, and long-tail methods address distributional skew through local resampling, reweighting, or distillation under synchronous participation. None of them directly addresses the coupled failure studied here, in which rare-label clients both enter the asynchronous aggregation buffer less frequently and receive insufficient influence when they do.

This leads to the central question of this work: \textit{can
asynchronous FL preserve rare-label influence while retaining
the scalability benefits of arrival-driven aggregation?}
The target is not the traditional data-level class imbalance
problem, in which one class has fewer total samples, but rarity
\textit{in terms of client coverage}, which the asynchronous
aggregation mechanism actively suppresses even when sufficient
data exists across the federation. To answer this, we propose
FedRAW, a server-side aggregation method that preserves
rare-label influence directly at the aggregation pathway,
without changing local training, client objectives, or
communication protocols.

FedRAW combines two complementary mechanisms. First, client-level update deduplication addresses participation imbalance inside the server-side buffer: when a frequently arriving client submits multiple updates, the server keeps only its most recent update rather than allowing the same client to occupy multiple buffer positions. This prevents fast clients from repeatedly dominating the aggregation pool. Second, rare-label-aware weighting addresses influence imbalance during aggregation: clients carrying low-coverage labels receive higher aggregation weight once their updates are present in the buffer. Together, these mechanisms preserve both buffer presence and aggregation influence for rare-label clients.

\begin{figure}[t]
    \centering
    \includegraphics[width=\columnwidth]{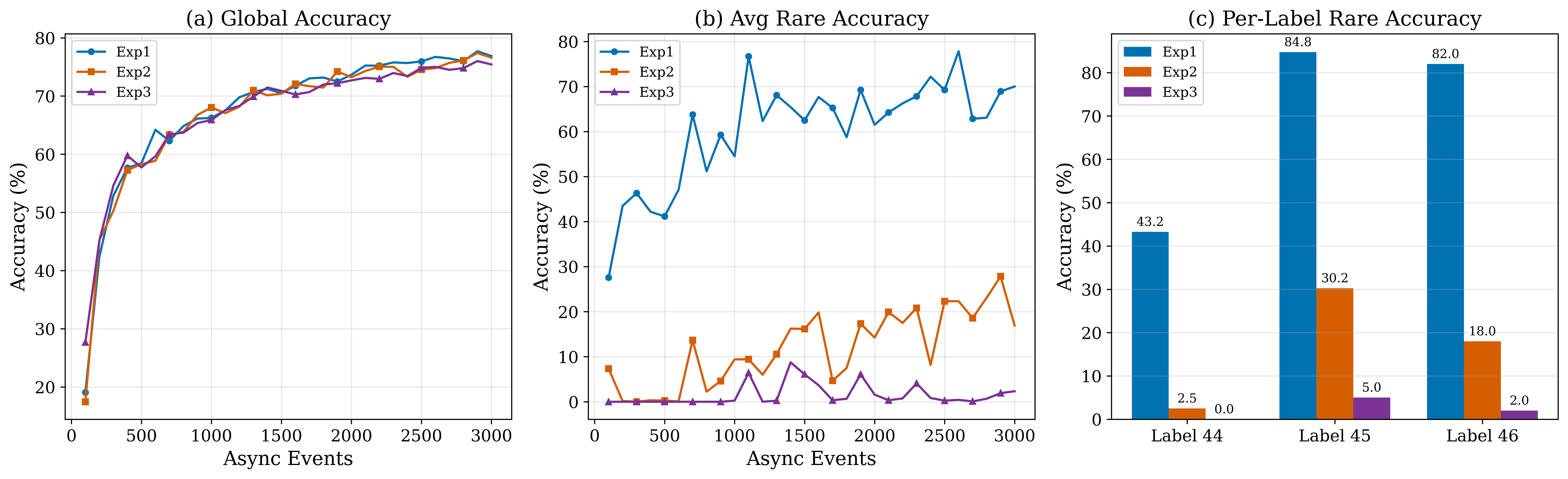}
    \caption{Silent rarity failure. \textbf{(a)}~Global accuracy
converges to $\sim$75--77\% across all three settings, masking the
failure. \textbf{(b)}~Rare-class accuracy reveals it: Exp2 collapses
to 16.9\% and Exp3 to 2.3\%. \textbf{(c)}~Per-label breakdown:
Labels 44, 45, 46 collapse to near-zero under Exp3, Label~44 to
0.0\%.}
    \label{fig:motivation}
\end{figure}

To the best of our knowledge, FedRAW is among the first
server-side asynchronous aggregation methods to explicitly
study and mitigate rare-label influence suppression caused
by the coupling of participation bias and aggregation bias
under label-coverage imbalance. Our main contributions are:
\begin{itemize}

\item We identify \textit{silent rarity failure} in buffer-based asynchronous FL, a hidden failure mode where rare-label clients contribute too weakly to the global model while overall accuracy appears unaffected. It arises from two coupled effects: rare-label clients enter the update buffer less frequently under heterogeneous participation, and once included, receive no compensating influence under uniform aggregation.

\item We propose FedRAW, a fully server-side aggregation method
that preserves rare-label influence without modifying local
training, client objectives, or communication protocols. It
combines client-level update deduplication, which prevents fast
clients from repeatedly dominating the buffer, with
rare-label-aware weighting, which amplifies the aggregation
influence of low-coverage clients, and we show that this increases
rare-label influence over uniform asynchronous aggregation.

\item We conduct extensive experiments on EMNIST Balanced,
CIFAR-10, HAM10000, and ISIC-2019, where FedRAW substantially
improves rare-label accuracy over asynchronous baselines at
comparable global accuracy and negligible server-side
computation.
\end{itemize}

\section{Related Work}

\subsection{Buffer-Based Asynchronous FL}
Asynchronous FL improves training efficiency by allowing the
server to update the global model as client updates arrive,
avoiding the synchronous waiting barrier. FedBuff~\cite{nguyen2022}
accumulates $K$ pseudo-gradients in a server-side buffer before
aggregation, while FedFa~\cite{yi2024} maintains a sliding window
of $K$ client model snapshots. Both target scalability and
convergence efficiency, but their aggregation is rarity-agnostic:
once an update enters the buffer, its influence does not depend on
whether the contributing client carries labels with low client
coverage, and frequently arriving clients may repeatedly occupy
buffer positions, reducing the effective presence of slower
clients. CA2FL~\cite{wang2024ca2fl} improves participation
coverage by maintaining a calibrated cached update delta $h_i$ for
each client, so that a global correction term
$\bar{h} = \frac{1}{N}\sum_i h_i$ lets all $N$ clients contribute
to every aggregation -- at the cost of $O(N)$ server-side memory,
each client's influence in $\bar{h}$ still remains
$1/N$.

\subsection{Long-Tail and Class Imbalance in FL}
Long-tail and class-imbalance methods in FL address distributional
skew mainly through data-side or local-training interventions:
Astraea~\cite{duan2019} uses mediator-based client rescheduling
and global data augmentation, BalanceFL~\cite{shuai2022} simulates
balanced data through knowledge distillation and feature-space
augmentation, CReFF~\cite{shang2022} corrects classifier bias by
retraining with federated features, FEDSLD~\cite{luo2022fedsld}
modifies the local objective with weighted cross-entropy for
medical imaging, and FedWeight~\cite{gao2026fedweight} reweights
the local loss by class-wise feature-space density. These methods
address sample-level or class-level imbalance but generally assume
synchronous or round-based participation, and do not model the
arrival-driven mechanism through which rare-label clients may
enter an asynchronous buffer less frequently and receive
insufficient influence once included.

\subsection{Fairness and Performance Parity in FL}
Fairness-aware FL studies disparity across clients or demographic
groups. q-FedAvg~\cite{li2020qfed} upweights high-loss clients to
improve device-level accuracy parity, and
FairFed~\cite{ezzeldin2023} adjusts aggregation weights using
local and global fairness metrics. Their objectives are not
defined in terms of label coverage, however: a client may satisfy
device-level fairness criteria while still carrying rare labels
that remain underrepresented in the global model.

\subsection{Staleness-Aware Asynchronous FL}
A separate line of asynchronous FL research focuses on staleness.
SyncFed~\cite{syncfed2025} quantifies staleness using NTP-based
timestamping and applies temporally informed aggregation weights,
and AFBS~\cite{afbs2025} clusters clients using encrypted label
distributions and selects high-value gradients to reduce
stale-update impact. These methods improve temporal reliability
but primarily ask how outdated an update is. FedRAW instead asks
whether the update comes from a client carrying low-coverage
labels, and whether such clients are structurally
underrepresented in the aggregation buffer.

\begin{table}[ht]
\centering
\caption{FedRAW positioning against related FL research categories.}
\label{tab:supp_positioning}
\small
\resizebox{\columnwidth}{!}{%
\begin{tabular}{lccccc}
\toprule
\textbf{Research category}
& \textbf{Sample-count}
& \textbf{Async}
& \textbf{Buffer}
& \textbf{Server-}
& \textbf{Rare-label} \\
& \textbf{imbalance}
& \textbf{arrival bias}
& \textbf{dominance}
& \textbf{side only}
& \textbf{influence} \\
\midrule
Buffer-based async FL
& No & Partial$^{\dagger}$ & No & Yes & No \\

Cached-update async FL
& Partial & Partial$^{\parallel}$ & Partial & Yes & No \\

Long-tail / class-imbalance FL
& Yes & No & No & Usually no & No$^{\ddagger}$ \\

Fairness / performance-parity FL
& Indirect & No & No & Partial & No$^{\S}$ \\

Staleness-aware async FL
& No & Partial$^{\P}$ & No & Partial & No \\

\textbf{FedRAW (Ours)}
& \textbf{No$^{\#}$}
& \textbf{Yes}
& \textbf{Yes}
& \textbf{Yes}
& \textbf{Yes} \\
\bottomrule
\end{tabular}%
}
\\[2pt]
\raggedright
\footnotesize
$^{\dagger}$addresses aggregation efficiency, not label-coverage-aware influence;
$^{\ddagger}$targets sample-count imbalance, a distinct axis from label coverage, without modelling arrival-driven buffer dynamics;
$^{\S}$targets client-level or demographic fairness rather than rare-label coverage;
$^{\P}$compensates for stale updates rather than rare-label underrepresentation;
$^{\parallel}$improves participation coverage via cached updates but does not weight low-coverage labels;
$^{\#}$targets label-coverage (system-induced) rather than sample-count imbalance.
\end{table}

Overall, prior work (Table~\ref{tab:supp_positioning}) addresses important but separate challenges,
and none directly addresses the coupled failure studied here, where
rare-label clients may both enter the buffer less frequently and
receive insufficient influence once they do.

\section{FedRAW: Federated Rarity-Aware Weighting}

\begin{figure}[t]
  \centering
  \includegraphics[width=1\columnwidth]{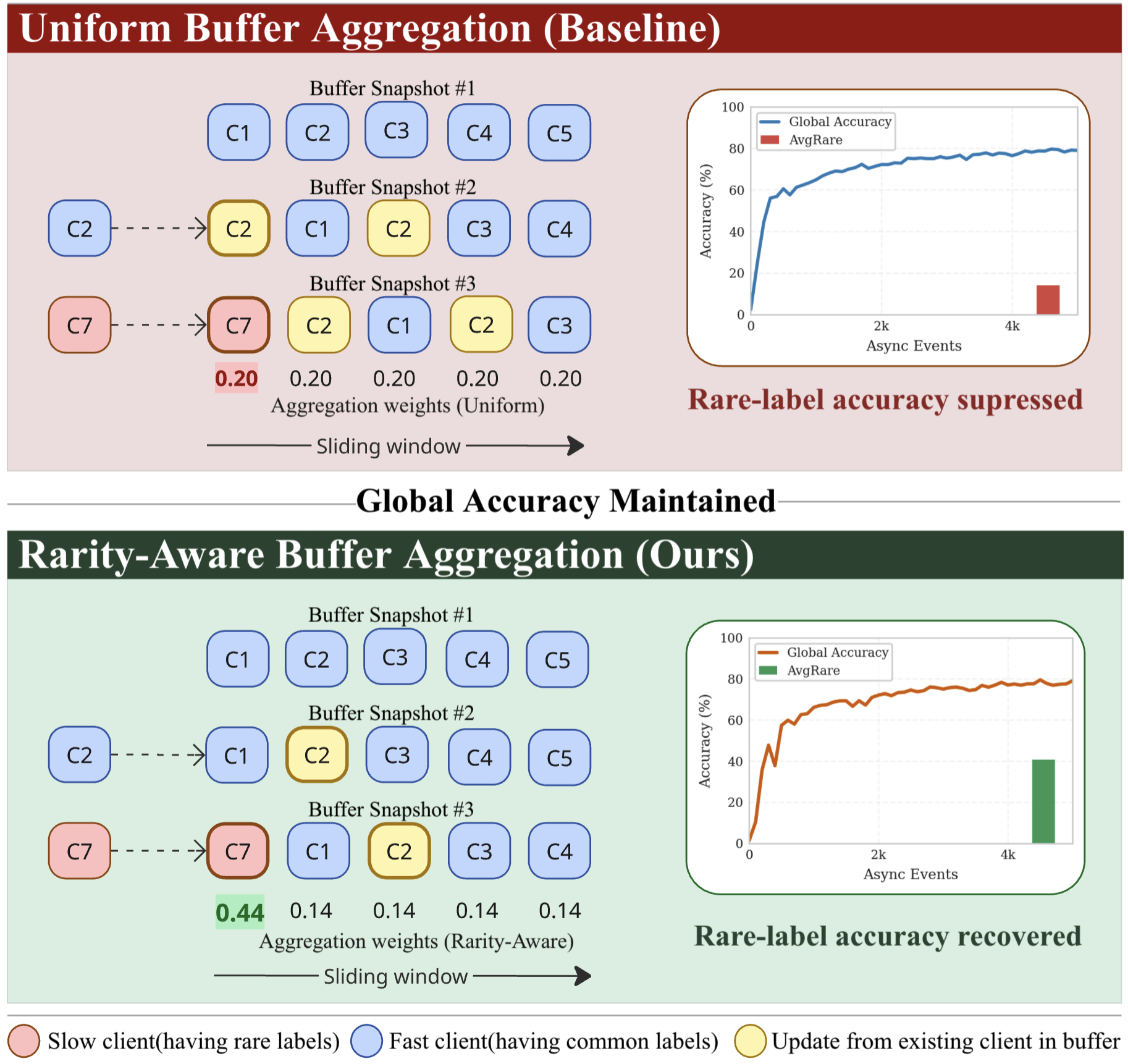}
\caption{
Comparison of uniform buffer aggregation (baseline) and FedRAW
across three buffer snapshots ($K=5$).
\textbf{Top:} Uniform buffering admits repeated updates from the
same client; hence, a new update from $C_2$ occupies a second slot.
When rare-label client $C_7$ arrives, uniform weighting
($w_i=1/K$) gives all buffered updates equal influence, allowing
frequent clients to dominate.
\textbf{Bottom:} FedRAW maintains at most one entry per client by
replacing $C_2$ in-place. Rarity-aware aggregation further assigns
greater weight to the rare-label client
($w_{C_7}=0.44$ vs.\ $0.14$ for each common-label client),
improving rare-label accuracy while maintaining global accuracy.
}
  \label{fig:fedraw_concept}
\end{figure}

FedRAW is a server-side correction to the two sources of silent
rarity failure: buffer deduplication improves buffer presence,
while rarity-aware weighting improves influence conditional on
buffer presence. Fig.~\ref{fig:fedraw_concept} illustrates the
difference from uniform buffer aggregation, and
Section~\ref{sec:analysis} then formalizes this intuition and
shows that FedRAW increases rare-label client influence relative
to uniform asynchronous aggregation while preserving convergence
under bounded staleness.

\subsection{Client-Level Buffer Deduplication}
In buffer-based asynchronous FL, the server maintains a sliding update buffer of capacity $K$ holding newly arrived client updates before aggregation. Under heterogeneous participation, fast clients may submit repeatedly and occupy multiple buffer positions, reducing the effective buffer presence of slower clients, including rare-label clients.

FedRAW introduces client-level buffer deduplication to prevent this repeated occupation. The server enforces a one-slot-per-client rule inside the update buffer. When a new update arrives from client $i$, the server first checks whether client $i$ already has an update in the buffer. If so, the previous update is replaced by the newly arrived update. If client $i$ is not currently present in the buffer, the new update is appended to the buffer; when the buffer exceeds capacity $K$, the oldest entry is evicted as in a standard sliding window.

The mechanism changes only how the server maintains its asynchronous update buffer, not client-side training or communication. By preventing fast clients from occupying multiple buffer positions, deduplication preserves buffer presence for distinct clients and directly mitigates participation bias, increasing the opportunity for rare-label clients to be represented in the aggregation buffer.

\subsection{Rare-Label-Aware Weighted Aggregation}

Deduplication improves the opportunity of rare-label clients to appear in the server-side buffer, but under uniform asynchronous aggregation every buffered update receives the same weight, even when some clients carry labels observed by only a few clients. FedRAW therefore introduces rare-label-aware weighting to preserve aggregation influence for clients carrying low-coverage labels.

FedRAW uses a one-time client label summary at initialization. It never accesses client data. We assume legitimate for clients to share local label distributions, which is often the practice in cross-silo FL such as medical applications~\cite{luo2022fedsld}, and in prior work where devices report local distribution information~\cite{duan2019}. The summary is used only to compute per-client rarity scores; no label information is exchanged during training. For each label $l$, we define its client coverage as the number of clients holding at least one sample of that label:
\begin{equation}
\label{eq:coverage}
\mathrm{coverage}(l) =
\left|{i : \mathrm{freq}(l,i) > 0}\right|,
\end{equation}
where $\mathrm{freq}(l,i)$ denotes the fraction of client $i$'s local data belonging to label $l$. A label with low client coverage is considered rare because it is represented by only a small number of clients.

The rarity score of client $i$ is computed by weighting each local label according to the inverse of its client coverage:
\begin{equation}
\label{eq:score}
S_i =
\sum_l
\frac{\mathrm{freq}(l,i)}{\mathrm{coverage}(l)}.
\end{equation}
The score is large when a substantial portion of a client's local data belongs to labels present in only a few clients, and small when its labels are widely shared across the federation. Scores are computed once at initialization and remain fixed during training.

Given the current server-side buffer $B$, FedRAW normalizes the rarity scores of the buffered clients to obtain aggregation weights:
\begin{equation}
\label{eq:weight}
w_i =
\frac{S_i}{\sum_{j \in B} S_j},
\qquad i \in B.
\end{equation}
The global model is then updated as:
\begin{equation}
\label{eq:aggregation}
\theta_{\mathrm{new}} =
\sum_{i \in B} w_i \theta_i .
\end{equation}

Clients carrying low-coverage labels therefore receive higher aggregation influence than under uniform $1/K$
averaging, especially when those labels form a meaningful portion of their local data. This directly mitigates
aggregation bias: once a rare-label client enters the buffer, its contribution is no longer interchangeable with
updates from clients carrying widely represented labels.

\begin{algorithm}[!t]
\caption{\textit{FedRAW}-Server}
\label{alg:fedraw}
\begin{algorithmic}[1]
\renewcommand{\algorithmicrequire}{\textbf{Input:}}
\renewcommand{\algorithmicensure}{\textbf{Output:}}
\REQUIRE Number of clients $N$; buffer size $K$; one-time client label-presence sets $\{\mathcal{Y}_i\}_{i=1}^{N}$; server-side buffer $B$
\ENSURE Global model $\theta_g$

\STATE Initialize $t=0$, global model $\theta_g^0$, and empty buffer $B$
\STATE Compute $\mathrm{coverage}(l)=|\{i:l\in\mathcal{Y}_i\}|$ for each label $l$
\STATE Compute rarity score $S_i = \sum_{l \in \mathcal{Y}_i} \mathrm{freq}(l,i)/\mathrm{coverage}(l)$ for all $i$
\STATE Broadcast $\theta_g^0$ to clients
\REPEAT
\IF{server receives update $\theta_i^t$ from client $i$}
\IF{client $i$ already has an entry in $B$}
\STATE Replace the existing entry of client $i$ in $B$ with $\theta_i^t$
\ELSE
\STATE Append $(i,\theta_i^t)$ to $B$
\IF{$|B|>K$}
\STATE Evict the oldest entry from $B$
\ENDIF
\ENDIF
\IF{$|B|=K$}
  \STATE $Z \leftarrow \sum_{j\in B} S_j$
  \STATE $\theta_g^{t+1} \leftarrow \sum_{j\in B} \dfrac{S_j}{Z}\theta_j$
  \STATE Broadcast $\theta_g^{t+1}$ to clients
\ENDIF
\STATE $t \leftarrow t+1$
\ENDIF
\UNTIL{convergence criterion is met}
\end{algorithmic}
\end{algorithm}

Algorithm~\ref{alg:fedraw} summarizes the server-side procedure, where $K$ is the buffer capacity, $B$ the server-side update buffer, $\mathcal{Y}_i$ the one-time label-presence set reported by client $i$, $S_i$ its rarity score, and $\theta_i^t$ the model update received from client $i$ at event $t$.
 
FedRAW adds $O(K)$ work per aggregation event -- one pass over the
buffer to sum rarity scores and compute normalised weights. The
scores $S_i$ are computed once at initialisation in $O(NL)$ time,
where $L$ is the number of labels, and never recomputed during
training, while deduplication requires a single hash-map lookup
per arriving update, adding $O(1)$ overhead per event.

\section{Analysis}
\label{sec:analysis}

This section formalizes the conditions under which rare-label clients lose aggregation influence under uniform buffer-based asynchronous FL, and under which FedRAW mitigates that loss, focusing on client influence at the server-side aggregation pathway.

\subsection{Participation Bias under Heterogeneous Arrival}

Let client $i$ require $s_i>0$ seconds to produce an update. Under arrival-driven asynchronous training, clients with smaller update time submit updates more frequently, so we model the arrival rate of client $i$ as $r_i=1/s_i$ and the probability that the next arriving update comes from client $i$ as $P(i)=r_i/\sum_{j=1}^{N} r_j$.

\begin{proposition}[Participation disadvantage]
If rare-label clients are slower on average than common-label clients, i.e.,
$\mathbb{E}[s_R]>\mathbb{E}[s_C]$, then their expected arrival probability is
smaller, $P(R) < P(C)$, and they consequently enter the server-side update buffer
less frequently under arrival-driven aggregation.
\end{proposition}

Since $r_i=1/s_i$ decreases monotonically in $s_i$, slower clients have smaller arrival rates, so rare-label clients that are slower on average have lower normalized arrival probability. This produces participation bias before aggregation even occurs. $\square$

\subsection{Aggregation Bias under Uniform Buffer Weighting}

Uniform buffer aggregation creates a second disadvantage: once a rare-label update is present, it receives the same weight as every other buffered update. Let $B$ be the aggregation buffer of size $K$ and $w_i$ the weight of client $i$ when present in $B$, so $w_i=1/K$ under uniform aggregation.

Define the per-event influence of client $i$ as $I_i = P(i\in B)\cdot w_i$.

\begin{proposition}[Aggregation disadvantage]
Under uniform buffer aggregation, if a rare-label client has lower buffer presence probability than a common-label client, then its per-event influence is also lower:
\begin{equation}
\label{eq:uniform_disadvantage}
I_R^{\mathrm{uniform}}=P(R\in B)\frac{1}{K}<P(C\in B)\frac{1}{K}=I_C^{\mathrm{uniform}}.
\end{equation}
\end{proposition}

Because the same weight $1/K$ is assigned to all buffered updates, any difference in buffer presence probability transfers directly to aggregation influence, and rare-label clients enter the buffer less frequently under heterogeneous arrival. $\square$

\subsection{Influence Preservation under FedRAW}

FedRAW addresses the two sources of influence loss at the server side. Client-level update deduplication improves buffer presence by preventing frequently arriving clients from occupying multiple buffer positions. Rare-label-aware weighting improves aggregation influence once a rare-label update is present.

\begin{proposition}[Effect of update deduplication]
Let $B^{\mathrm{uniform}}$ be a standard sliding buffer and $B^{\mathrm{dedup}}$ be the FedRAW buffer with at most one update per client. If a frequently arriving client occupies multiple positions in $B^{\mathrm{uniform}}$, then deduplication weakly increases the number of distinct clients represented in the buffer and reduces duplicate-slot domination by fast clients.
\end{proposition}

In a standard sliding buffer, repeated arrivals from the same fast client can occupy multiple positions; FedRAW replaces that client's existing buffered update instead of appending a duplicate, so duplicate slots are removed and capacity is reserved for distinct clients. This weakly improves the opportunity for slower clients, including rare-label clients, to be represented. $\square$

\begin{proposition}[Effect of rare-label-aware weighting]
For each client $i$, FedRAW computes a client rarity score
\begin{equation}
S_i=\sum_{l\in \mathcal{Y}_i}\frac{\mathrm{freq}(l,i)}{\mathrm{coverage}(l)},
\end{equation}
where $\mathcal{Y}_i$ is the label-presence set of client $i$. 
Given buffer $B$, FedRAW assigns
\begin{equation}
w_i^{\mathrm{FedRAW}}=
\frac{S_i}{\sum_{j\in B}S_j}.
\end{equation}
For any buffered rare-label client $r$ and common-label client $c$, if $S_r>S_c$, then
\begin{equation}
w_r^{\mathrm{FedRAW}}>w_c^{\mathrm{FedRAW}}.
\end{equation}
Moreover, if $S_r$ is larger than the average rarity score of the buffered clients, then $w_r^{\mathrm{FedRAW}}>1/K$.
\end{proposition}

Because the weights are normalized rarity scores, a larger score yields a larger aggregation weight, and a score above the buffer average yields a weight exceeding the uniform $1/K$ -- directly compensating for aggregation bias once the rare-label client is in the buffer. $\square$

Together, these propositions show that when rare clients carry above-average rarity scores, deduplication mitigates duplicate-slot domination while rare-label-aware weighting increases the influence of low-coverage clients during aggregation.

\subsection{Convergence under Bounded Staleness}
\label{sec:convergence}

We show that rarity-aware weighting preserves convergence. The central point is
that a \emph{normalized} weighted aggregator performs stochastic gradient descent
not on the uniform objective $f(\theta)=\frac{1}{m}\sum_{i=1}^{m}F_i(\theta)$, but on
an \emph{influence-weighted} objective $\bar F$ determined by the rarity weights.
FedRAW attains the standard non-convex stationary-point guarantee for $\bar F$: the
average squared gradient $\frac1T\sum_t\mathbb{E}\|\nabla\bar F(\theta_g^t)\|^2$ is
bounded by a vanishing $\mathcal{O}(1/\sqrt{TQ})$ optimization-and-variance term plus a
bounded-staleness constant controlled by the drift assumption, matching the
buffered-FL guarantees of FedFa~\cite{yi2024}/FedBuff~\cite{nguyen2022} (to which the
bound reduces under uniform weights and uniform participation). Full
proofs are in Supplementary Section~S-VIII.

\textit{Assumptions.}
We adopt the standard assumptions of buffered asynchronous
FL~\cite{nguyen2022,yi2024}: each $F_i$ is $L$-smooth; stochastic gradients are
unbiased with bounded local variance $\sigma_\ell^2$; gradients and client
heterogeneity are bounded ($\|\nabla F_i\|^2\le G$,
$\sigma^2:=\sigma_\ell^2+\sigma_g^2+G$); and staleness satisfies
$\tau_i(t)\le\tau_{\max}$. FedRAW additionally assumes
(i) fixed rarity scores computed once at initialization,
(ii) stationary (possibly non-uniform) client participation,
(iii) the standard bounded-drift condition
$\mathbb{E}\|\theta_g^t-\theta_g^{t-\tau}\|^2\le\Psi_\tau$, and
(iv) bounded buffer-sampling variance
$\mathbb{E}\|\sum_{i\in B_t}w_i\nabla F_i-\nabla\bar F\|^2\le\sigma_B^2\le4\rho G$.

\textit{Influence-weighted objective.}
Let
$\pi_i:=\mathbb{E}[\,w_i\indicator\{i\in B_t\}\,]$
denote the expected aggregation influence of client $i$, where
$w_i=S_i/\sum_{j\in B_t}S_j$.
Since $w_i\ge0$ and $\sum_{i\in B_t}w_i=1$, the influences satisfy
$\sum_i\pi_i=1$, making
$\bar F(\theta)=\sum_i\pi_iF_i(\theta)$
a valid $L$-smooth objective.
Taking expectation over the buffer first yields
$\mathbb{E}[\sum_{i\in B_t}w_ig_i]=\nabla\bar F$,
so FedRAW optimizes $\bar F$ rather than the uniform objective $f$.
Larger rarity scores increase a client's \emph{conditional} aggregation weight
when present in the buffer, whereas the unconditional influence
$\pi_i$ depends jointly on rarity scores and participation probabilities.
Consequently, the convergence guarantee is with respect to stationary points of
$\bar F$; $\bar F=f$ only under uniform weights and uniform participation.

\textit{Weight non-uniformity.} Let
$w_{\mathrm{sq}}:=\sup_t\sum_{i\in B_t}w_i^2$; by Cauchy--Schwarz
$w_{\mathrm{sq}}\ge 1/K$, with equality iff the weights are uniform. We define
$\rho:=K\,w_{\mathrm{sq}}\in[1,K]$ as a bounded non-uniformity factor, so $\rho=1$
recovers uniform weighting and larger $\rho$ reflects more aggressive rare-label
amplification.

\begin{theorem}[Convergence of FedRAW]
\label{thm:main}
Under the assumptions above and the step-size condition
$\eta_\ell^{(q)}Q\le 1/(4L)$ for all local steps $q$ (with server rate $\eta_g=1$, i.e.\
parameter averaging as in Algorithm~1), the FedRAW iterates satisfy
\begin{align}
\label{eq:theorem}
\frac{1}{T}\sum_{t=0}^{T-1}\!\mathbb{E}\big\|\nabla\bar F(\theta_g^{t})\big\|^2
&\le 
\frac{8\,\bar F^{\star}}{\alpha(Q)\,T}
+ 24L^2 Q\,\beta(Q)\big(\tau_{\max}^2+1\big)\sigma^2 \notag\\
&\quad + \rho\cdot\frac{8L\,\beta(Q)}{\alpha(Q)}\cdot\Big(\frac{\sigma_\ell^2}{K}+4G\Big) \notag\\ 
&\quad + C_{\mathrm{drift}}\,\sigma^2\tau_{\max}^2,
\end{align}
where $\bar F^{\star}=\bar F(\theta_g^{0})-\inf_\theta\bar F(\theta)$,
$\alpha(Q)=\sum_{q}\eta_\ell^{(q)}$, $\beta(Q)=\sum_{q}(\eta_\ell^{(q)})^2$, and
$C_{\mathrm{drift}}=\mathcal{O}(1)$ is the bounded-drift constant (Assumption on global
drift). The
$\rho$-scaled term collects the local-noise variance $\sigma_\ell^2/K$ and the
buffer-sampling variance $\sigma_B^2\le4\rho G$; the staleness terms are $\rho$-free.
The first three terms vanish under the diminishing schedule; the last is the
irreducible bounded-staleness constant of asynchronous buffered FL. All leading
constants are derived in the supplementary material; only a single $\mathcal{O}(1)$
step-size factor is left in order notation.
\end{theorem}

With the diminishing schedule
$\eta_\ell=\mathcal{O}(1/\sqrt{TQ})$ (and $\eta_g=1$),
the first three terms of~\eqref{eq:theorem} decrease as
$\mathcal{O}(1/\sqrt{TQ})$, while the remaining
$\mathcal{O}(\sigma^2\tau_{\max}^2)$ term is the standard bounded-staleness
neighborhood of buffered asynchronous FL. Importantly, this neighborhood is
independent of the rarity-weighting factor $\rho$; $\rho$ affects only the
variance term through the non-uniform aggregation weights. Exact constants are
given in Supplementary Section~S-VIII.

\noindent\textbf{Remarks.}
\emph{(i) Cost of rarity weighting.}
Rarity-aware aggregation increases only the variance term through the bounded
factor $\rho\in[1,K]$; the staleness term remains unchanged and carries no
$\rho$ dependence.
\emph{(ii) Consistency.}
Under uniform weights and uniform participation,
$\pi_i=1/m$, $\bar F=f$, and Theorem~\ref{thm:main} reduces to the
FedFa/FedBuff convergence guarantee. Deduplication does not introduce an
additional optimization term; it acts only by modifying the stationary
participation distribution.

\section{Evaluation Results}
\subsection{Experimental Setup}

\subsubsection{Datasets and Federated Partitions}

We evaluate FedRAW on four datasets covering controlled and real-world settings. Rarity here is defined by \emph{label coverage}: a rare label is held by only a small number of clients, which is exactly the condition under which rare-label clients appear less often in the asynchronous buffer and receive insufficient aggregation influence.

\textbf{EMNIST Balanced}~\cite{cohen2017emnist} contains 112,800 training and 18,800 test images across 47 classes; we use 30 clients with a controlled long-tail non-IID partition in which labels {44, 45, 46} are rare and each held by only 2 clients, while common labels are held by 20 clients each. \textbf{CIFAR-10}~\cite{krizhevsky2009cifar} contains 50,000 training and 10,000 test images across 10 classes, partitioned with the same client-coverage structure, with airplane (0) and frog (6) designated rare under correlated speed. \textbf{HAM10000}~\cite{tschandl2018} contains 10,015 dermoscopy images across 7 diagnostic classes, of which DF (115 images, 1.1\%) and VASC (142 images, 1.4\%) are naturally rare, a 58:1 imbalance ratio relative to the dominant NV class. \textbf{ISIC~2019}~\cite{isic2019} contains 25,331 dermoscopy images across 8 diagnostic classes, with VASC (253 images) and DF (239 images) designated rare, a 54:1 imbalance ratio. Both medical datasets follow the same 30-client structure, with 2 clients per rare class and 20 clients per common class.

Unless stated otherwise, all experiments use $N=30$ clients and
buffer size $K=10$, with 2 rare clients per rare label, 5{,}000
asynchronous events on EMNIST Balanced and 10{,}000 on the other
datasets, 2 local epochs, learning rate 0.01, batch size 256, and
seeds 42, 123, and 456. Under the correlated-speed setting,
rare-label clients draw update times from $[1.5,3.0]$\,s and
common-label clients from $[0.5,1.5]$\,s. EMNIST Balanced uses a
lightweight CNN, CIFAR-10 a ResNet-18~\cite{he2016resnet}, and
HAM10000/ISIC-2019 an ImageNet~\cite{deng2009imagenet}-pretrained
ResNet-18~\cite{he2016resnet}. For q-FedAvg we set $q=5$.

\begin{table*}[t]
\centering
\caption{Performance comparison across datasets under correlated 
speed setting ($K=10$, $N=30$), averaged over 3 seeds. 
Methods above the divider are Synchronous baselines;
below are Asynchronous. \textbf{Bold}: best asynchronous 
result; \underline{underline}: second best.}
\label{tab:main_results_all}
\scriptsize
\renewcommand{\arraystretch}{0.95}
\setlength{\tabcolsep}{2pt}
\resizebox{\textwidth}{!}{%
\begin{tabular}{llcccccccc}
\toprule
\textbf{Dataset} &
\textbf{Algorithm} &
\textbf{GlobalAcc} &
\textbf{AvgRare} &
\textbf{Worst-10\%} &
\textbf{MacroF1} &
\textbf{RareF1} &
\textbf{RareF2} &
\textbf{LocalRare} &
\textbf{LocalCommon} \\
\midrule

\multirow{6}{*}{\rotatebox[origin=c]{90}{\parbox[c]{1.4cm}{\centering EMNIST}}}
& FedAvg   & $81.8\pm0.1$ & $38.2\pm0.8$ & $70.6\pm0.4$ & $81.1\pm0.1$ & $49.7\pm1.3$ & $42.0\pm0.9$ & $74.4\pm0.0$ & $85.2\pm0.1$ \\
& q-FedAvg & $50.3\pm0.5$ & -- & $40.6\pm0.7$ & $47.6\pm0.5$ & -- & -- & $41.4\pm1.2$ & $53.3\pm0.5$ \\
\cmidrule(lr){2-10}
& CA2FL    & $\mathbf{78.2\pm0.1}$ & \ul{$25.7\pm0.2$}     & $\mathbf{65.4\pm0.3}$ & $\mathbf{76.9\pm0.1}$ & \ul{$36.9\pm0.4$}     & \ul{$29.3\pm0.2$}     & \ul{$68.0\pm0.7$}     & \ul{$82.0\pm0.4$}     \\
& FedBuff  & $77.1\pm1.0$          & $13.4\pm0.5$          & $61.1\pm1.1$          & $75.0\pm1.4$          & $21.1\pm0.3$          & $15.7\pm0.4$          & $65.0\pm0.8$          & $81.6\pm1.1$          \\
& FedFa  & \ul{$77.5\pm0.8$}     & $6.0\pm4.0$           & $62.0\pm0.5$          & $75.3\pm1.0$    & $10.5\pm6.3$          & $7.2\pm4.7$           & $63.5\pm1.1$          & $\mathbf{82.3\pm1.2}$ \\
& \textbf{FedRAW} & $76.6\pm1.1$          & $\mathbf{30.2\pm4.2}$ & \ul{$63.4\pm1.5$}     & \ul{$75.3\pm1.2$}          & $\mathbf{37.9\pm3.3}$ & $\mathbf{32.7\pm3.8}$ & $\mathbf{68.4\pm0.3}$ & $80.0\pm1.5$          \\
\midrule

\multirow{6}{*}{\rotatebox[origin=c]{90}{\parbox[c]{1.4cm}{\centering CIFAR-10}}}
& FedAvg   & $57.3\pm1.4$  & $53.2\pm10.1$ & $98.3\pm1.1$ & $57.0\pm1.1$ & $49.5\pm5.1$  & $51.6\pm7.8$  & $55.4\pm7.6$  & $54.9\pm0.7$ \\
& q-FedAvg  & $9.9\pm0.9$   & $40.0\pm12.0$ & --  & $2.8\pm1.4$  & $8.0\pm1.2$   & $15.4\pm2.9$  & $26.3\pm8.7$  & $3.0\pm4.7$  \\
\cmidrule(lr){2-10}
& CA2FL   & $\mathbf{55.7\pm1.0}$ & $\mathbf{44.9\pm4.6}$  & $\mathbf{88.8\pm5.8}$  & $\mathbf{55.2\pm0.7}$ & $\mathbf{45.4\pm3.7}$  & $\mathbf{45.0\pm4.2}$  & $\mathbf{49.9\pm5.2}$  & \ul{$55.5\pm0.3$}     \\
& FedBuff   & $52.2\pm0.8$          & $20.5\pm3.7$           & $62.9\pm5.4$           & $50.6\pm0.3$          & $29.6\pm3.7$           & $23.3\pm3.7$           & $35.0\pm2.4$           & $\mathbf{57.7\pm3.0}$ \\
& FedFa   & $50.1\pm1.4$          & $22.0\pm9.3$           & $62.7\pm9.0$           & $48.4\pm1.6$          & $28.2\pm7.5$           & $23.9\pm8.7$           & $34.3\pm5.2$           & $54.9\pm1.3$          \\
& \textbf{FedRAW} & \ul{$52.8\pm2.5$}     & \ul{$33.8\pm19.0$}     & \ul{$76.9\pm21.0$}     & \ul{$51.6\pm3.3$}     & \ul{$35.9\pm15.4$}     & \ul{$34.3\pm17.8$}     & \ul{$42.1\pm11.9$}     & $54.8\pm2.5$          \\
\midrule

\multirow{6}{*}{\rotatebox[origin=c]{90}{\parbox[c]{1.4cm}{\centering HAM10000}}}
& FedAvg   & $81.0\pm1.1$ & $64.8\pm6.9$  & $61.4\pm8.1$ & $66.0\pm3.1$ & $66.4\pm3.9$  & $65.2\pm5.8$  & $74.2\pm8.3$  & $74.4\pm2.9$ \\
& q-FedAvg  & $63.5\pm2.5$ & $20.3\pm5.1$  & $33.2\pm4.9$ & $39.1\pm3.8$ & $27.1\pm7.1$  & $22.5\pm5.7$  & $53.6\pm9.0$  & $59.3\pm0.5$ \\
\cmidrule(lr){2-10}
& CA2FL        & $\mathbf{76.2\pm1.5}$ & $\mathbf{45.3\pm15.5}$ & $\mathbf{61.2\pm10.5}$ & \ul{$59.1\pm4.0$}     & $\mathbf{50.0\pm10.0}$ & $\mathbf{46.8\pm13.8}$ & $\mathbf{69.1\pm8.6}$  & $72.1\pm4.2$          \\
& FedBuff     & $75.1\pm1.3$          & $12.9\pm15.6$          & $49.7\pm7.2$           & $48.0\pm5.2$          & $16.1\pm17.8$          & $13.9\pm16.4$          & $65.3\pm8.0$           & \ul{$73.4\pm2.8$}     \\
& FedFa    & $75.3\pm3.1$          & $12.8\pm14.2$          & $52.5\pm4.0$           & $48.8\pm5.2$          & $15.2\pm14.0$          & $13.5\pm14.3$          & $66.0\pm11.6$          & $\mathbf{73.6\pm2.8}$ \\
& \textbf{FedRAW} & \ul{$75.7\pm2.0$}     & \ul{$43.2\pm17.7$}     & \ul{$57.9\pm6.6$}      & $\mathbf{59.5\pm4.9}$ & \ul{$49.9\pm14.0$}     & \ul{$45.2\pm16.7$}     & \ul{$67.5\pm7.8$}      & $73.2\pm2.0$          \\
\midrule

\multirow{6}{*}{\rotatebox[origin=c]{90}{\parbox[c]{1.4cm}{\centering ISIC-2019}}}
& FedAvg   & $73.0\pm0.9$  & $64.9\pm3.2$  & $77.8\pm14.2$ & $60.4\pm1.0$ & $62.3\pm2.9$  & $63.8\pm3.1$  & $74.6\pm2.2$  & $70.0\pm1.8$  \\
& q-FedAvg  & $24.2\pm21.5$ & $12.3\pm13.3$ & --   & $7.0\pm4.5$  & --  & $1.5\pm0.9$   & $23.4\pm20.2$ & $18.5\pm15.5$ \\
\cmidrule(lr){2-10}
& CA2FL          & $67.1\pm0.4$          & \ul{$42.0\pm0.7$}     & $64.3\pm2.1$          & \ul{$52.9\pm0.4$}     & \ul{$45.3\pm1.4$}     & \ul{$43.0\pm0.9$}     & \ul{$65.4\pm0.7$}     & $65.8\pm1.0$          \\
& FedBuff    & $\mathbf{69.4\pm1.2}$ & $22.3\pm2.5$          & $68.7\pm2.4$          & $51.0\pm0.2$          & $32.3\pm3.6$          & $25.4\pm2.8$          & $65.3\pm2.8$          & $\mathbf{68.9\pm1.7}$ \\
& FedFa     & $68.0\pm1.7$          & $29.6\pm1.9$          & $\mathbf{71.1\pm2.8}$ & $52.7\pm1.8$          & $39.1\pm6.3$          & $32.7\pm3.1$          & $63.9\pm3.3$          & $67.1\pm1.8$          \\
& \textbf{FedRAW} & \ul{$68.4\pm2.1$}     & $\mathbf{47.6\pm1.6}$ & \ul{$70.0\pm2.4$}     & $\mathbf{55.2\pm1.3}$ & $\mathbf{49.5\pm6.0}$ & $\mathbf{48.0\pm3.7}$ & $\mathbf{66.8\pm5.2}$ & \ul{$67.2\pm1.7$}     \\
\bottomrule
\end{tabular}%
}
\end{table*}


\subsubsection{Baselines}

We compare FedRAW with synchronous, asynchronous, cached-update,
and fairness-aware baselines. \textbf{FedAvg}~\cite{mcmahan2017}
is a synchronous reference training all clients each round.
\textbf{FedBuff}~\cite{nguyen2022} and \textbf{FedFa}~\cite{yi2024}
are the main asynchronous baselines, both using server-side
buffers with uniform, rarity-agnostic aggregation.
\textbf{CA2FL}~\cite{wang2024ca2fl} is a cached-update
asynchronous baseline that maintains previously received updates
so all clients contribute during aggregation.
\textbf{q-FedAvg}~\cite{li2020qfed} is a synchronous fairness-aware
baseline reweighting clients by their loss raised to power $q$; it
targets device-level loss parity rather than label-coverage
scarcity, and is included to test whether loss-based reweighting
incidentally recovers rare-label influence.

Additional experiments on scalability, alternative rare-label sets, hyperparameter sensitivity, Dirichlet non-IID partitioning, visually similar rare classes, and further ablations are provided in the supplementary material.

\subsubsection{Evaluation Metrics}

We report metrics in three categories. \textbf{GlobalAcc} is top-1
accuracy on the held-out global test set; because that set
contains many more common-class than rare-class samples, it can
remain high even when rare labels are poorly learned. To capture
this hidden failure we report \textbf{AvgRare}, the mean per-class
accuracy over rare labels only (e.g. {44,45,46} on EMNIST
Balanced), which is our primary metric for rare-label
preservation.

\textbf{F1-based rare-class detection.} \textbf{MacroF1} averages
F1 equally across classes, making it more sensitive to rare-class
failure than GlobalAcc since each class contributes equally
regardless of test-set frequency. \textbf{RareF1} averages F1 over
rare labels only, and \textbf{RareF2} is the corresponding
$\beta{=}2$ score, weighting recall more heavily -- useful when
missed rare-class detections cost more than false positives.

\textbf{Client-level utility.} \textbf{Worst-10\%} is the mean
accuracy of the bottom 10\% of clients ranked by local test
performance, measuring whether some clients remain structurally
disadvantaged. \textbf{Local Rare} and \textbf{Local Common} are
the accuracy of the global model on held-out local test sets of
rare-label and common-label clients, measuring utility from the
clients' own distributions rather than the global one.

\subsection{Performance Comparison with Baselines}

Table~\ref{tab:main_results_all} compares FedRAW with synchronous, asynchronous, cached-update, and fairness-aware baselines across EMNIST Balanced, CIFAR-10, HAM10000, and ISIC-2019. All results are averaged over three seeds unless stated otherwise. 
Fig.~\ref{fig:all_graphs} shows the corresponding global accuracy curves under correlated speed.

\begin{figure}[t]
    \centering
    \includegraphics[width=\columnwidth]{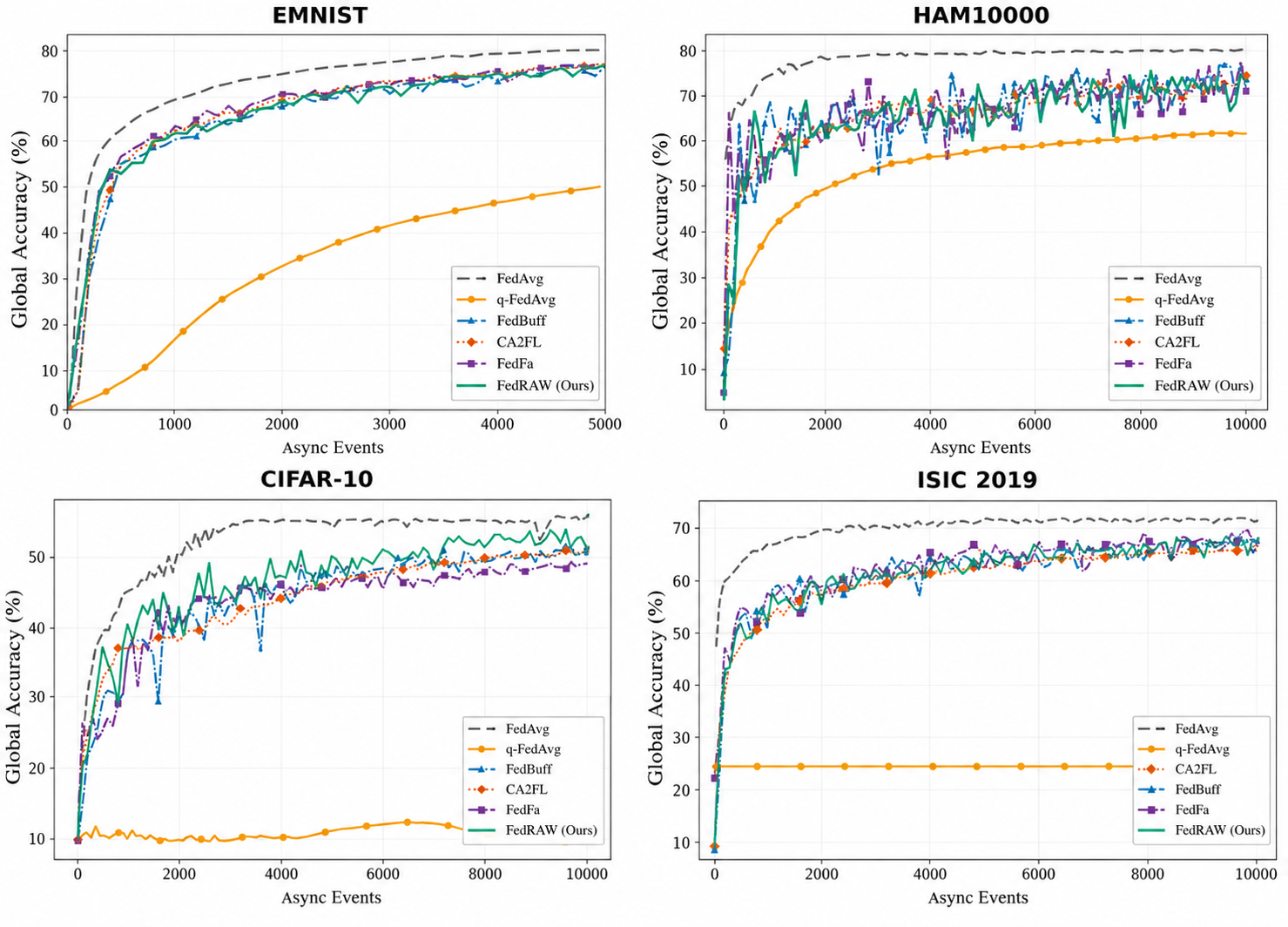}
    \caption{Global accuracy curves for all datasets results under correlated speed ($K$=10, $N$=30, seed~42).
    ~Global accuracy vs.\ async events.}
    \label{fig:all_graphs}
\end{figure}

\subsubsection{Rare-label accuracy}
FedRAW consistently improves rare-label accuracy over standard
asynchronous baselines: AvgRare reaches 30.2\% on EMNIST Balanced
against 6.0\% for FedFa and 13.4\% for FedBuff, and 33.8\% on
CIFAR-10 against 22.0\% and 20.5\%. CIFAR-10 is a substantially
harder representation-learning problem than EMNIST -- natural
images bring stronger visual complexity, richer texture variation,
and higher inter-class similarity -- which raises the difficulty of
recovering rare-label representations from few client updates. The gains are stronger on the medical datasets, where AvgRare improves over FedFa from 12.8\% to 43.2\% on HAM10000 and from 29.6\% to
47.6\% on ISIC-2019, confirming that the failure, and the
correction, generalize to realistic settings with naturally
underrepresented classes. The relatively high variance in the HAM10000 rare-label metrics is due to the
seed-dependent client data partition and randomness in client update arrival times.
Despite this variability, FedRAW consistently outperforms FedFa and FedBuff
across all three seeds, while remaining competitive with CA2FL; a seed-wise analysis is provided in the Supplementary Section S-VII.

\subsubsection{Rare-class detection quality}
The F1-based metrics confirm that FedRAW recovers rare-class
predictions rather than merely shifting overall accuracy. Over
FedFa, RareF1 improves from 10.5\% to 37.9\% on EMNIST Balanced,
28.2\% to 35.9\% on CIFAR-10, 15.2\% to 49.9\% on HAM10000, and
39.1\% to 49.5\% on ISIC-2019, with RareF2 improving in step
(7.2\% to 32.7\%, 23.9\% to 34.3\%, 13.5\% to 45.2\%, 32.7\% to
48.0\%). Since RareF2 weights recall more heavily, these gains
indicate fewer missed rare-class detections, the main failure mode
under uniform asynchronous aggregation.

\subsubsection{Global accuracy trade-off}
FedRAW improves rare-label performance while maintaining comparable global accuracy. Relative to FedFa, AvgRare improves by $+24.2$\,pp on EMNIST Balanced, $+11.8$\,pp on CIFAR-10, $+30.4$\,pp on HAM10000, and $+18.0$\,pp on ISIC-2019, while GlobalAcc changes by only $-0.9$\,pp, $+2.7$\,pp, $+0.4$\,pp, and $+0.4$\,pp respectively.

\subsubsection{Client-level utility}
The gains extend to the global model's usefulness for rare-label clients themselves. Relative to FedFa, LocalRare rises from 63.5\% to 68.4\% on EMNIST Balanced, 34.3\% to 42.1\% on CIFAR-10, and 63.9\% to 66.8\% on ISIC-2019, while on HAM10000 the LocalRare--LocalCommon gap narrows from 7.6\,pp to 5.7\,pp.

\subsubsection{Participation coverage versus aggregation influence}

CA2FL is a strong asynchronous baseline because its cached-update
mechanism guarantees that every client contributes to every aggregation.
Its weights remain uniform, so each client receives the same $1/30$
influence. On CIFAR-10, airplane and frog are rare by client coverage
but not by sample count (5{,}000 images each globally); rare clients
therefore hold substantially more data than common clients
(2{,}800--3{,}600 vs.\ 1{,}069 samples), explaining CA2FL's strong
AvgRare despite uniform weighting. On HAM10000, CA2FL and FedRAW are
comparable ($45.3\%\pm15.5$ vs.\ $43.2\%\pm17.7$), with high variance
reflecting rare-class test sets of only 23--28 images. On ISIC-2019, the
dominant class outnumbers rare classes by 54:1, causing common clients
to contribute 26/30 ($86.7\%$) of the uniform aggregation influence,
while rare clients contribute only 4/30 ($13.3\%$). FedRAW amplifies
rare-client influence above the uniform baseline, addressing aggregation
bias when it is the dominant bottleneck, while using $3\times$ less
server-side memory than CA2FL's per-client cache.

CA2FL and FedRAW therefore represent two complementary strategies:
participation preservation versus influence preservation. CA2FL preserves
client participation through cached updates, allowing clients that
participate less frequently to retain influence in subsequent
aggregations. In contrast, FedRAW focuses on increasing the influence of
low-coverage clients when their updates reach the aggregation buffer.
This distinction also helps explain why the relative performance of the
two methods varies across datasets.

\subsubsection{Device-level fairness is insufficient} 

q-FedAvg is a synchronous method; we evaluate it on the exact same data partitions used for our asynchronous experiments as a sanity check to verify whether device-level reweighting incidentally helps with label-coverage scarcity, not a like-for-like baseline. On EMNIST Balanced it
obtains only 0.1\% AvgRare and 0.1\% RareF1 while GlobalAcc drops
to 50.3\%, and on ISIC-2019 it collapses to 12.3\% AvgRare and
24.2\% GlobalAcc. A single-seed sweep over $q \in \{0.1, 1.0, 5.0\}$
(Supplementary Section~S-III) explores the sensitivity of q-FedAvg to q and shows that the observed degradation persists across the tested q values on seed 42: at $q{=}0.1$ the method approaches FedAvg behaviour with
32.8\% AvgRare, at $q{=}1.0$ both GlobalAcc and AvgRare degrade
further, and at $q{=}5.0$ AvgRare collapses to near zero.
Device-level fairness and label-coverage rarity are thus distinct
problems, and FedRAW's effectiveness is established against the
asynchronous baselines in
Table~\ref{tab:main_results_all}, not by this comparison.

\subsubsection{Effect of rare-client availability}
FedRAW is most effective when rare-label clients enter the buffer
often enough for rarity-aware weighting to act, which explains the
higher variance on CIFAR-10 and HAM10000, where rare-client
participation and rare-class sample counts are more limiting. On
CIFAR-10 it improves over FedFa but with larger variance, while
CA2FL benefits from guaranteed cached participation; the greater
visual complexity and inter-class overlap mean that amplified
rare-label signal must still compete against stronger
common-class representations. On HAM10000 the very small rare test
set makes AvgRare sensitive to a handful of misclassifications.
These results identify an operational boundary: FedRAW amplifies
rare-label signal once rare clients are in the buffer, but cannot
recover information from clients that rarely arrive, motivating
the analyses in Section~\ref{subsec:staleness}. Results for rare
classes with strong inter-class visual similarity are in
Supplementary Section~S-V.

\begin{table}[t]
\centering
\caption{Component ablation under correlated speed
($K$=10, $N$=30, 3~seeds).
BufPres: fraction of aggregations with at least one rare client
in the buffer.}
\label{tab:component_ablation}
\renewcommand{\arraystretch}{1.0}
\resizebox{\columnwidth}{!}{%
\begin{tabular}{llcccc}
\toprule
\textbf{Dataset} &
\textbf{Algorithm} &
\textbf{GlobalAcc(\%)} &
\textbf{AvgRare(\%)} &
\textbf{RareF1(\%)} &
\textbf{BufPres(\%)} \\
\midrule
\multirow{4}{*} {EMNIST}
  & FedFa               & $77.5 \pm 0.8$ & $6.0  \pm 4.0$ & $10.5 \pm 6.3$ & $60.2 \pm 6.5$ \\
  & FedRAW (DedupOnly)  & $77.0 \pm 0.5$ & $7.9  \pm 0.9$ & $13.9 \pm 1.5$ & $\mathbf{73.7 \pm 0.9}$$\uparrow$ \\
  & FedRAW (WeightOnly) & $77.9 \pm 0.3$ & $\mathbf{35.2 \pm 5.1}$$\uparrow$ & $\mathbf{44.6 \pm 5.1}$$\uparrow$ & $59.9 \pm 5.2$ \\
  & FedRAW (Both)       & $76.6 \pm 1.1$ & $30.2 \pm 4.2$$\uparrow$ & $37.9 \pm 3.3$$\uparrow$ & $72.7 \pm 3.5$$\uparrow$ \\
\midrule
\multirow{4}{*} {HAM10000}
  & FedFa               & $75.3 \pm 3.1$ & $12.8 \pm 14.2$ & $15.2 \pm 14.0$ & $46.4 \pm 2.2$ \\
  & FedRAW (DedupOnly)  & $76.6 \pm 1.6$ & $15.4 \pm 10.5$ & $22.2 \pm 13.2$ & $57.7 \pm 2.1$$\uparrow$ \\
  & FedRAW (WeightOnly) & $75.3 \pm 2.3$ & $\mathbf{53.8 \pm 11.5}$$\uparrow$ & $\mathbf{52.5 \pm 6.2}$$\uparrow$ & $46.6 \pm 2.1$ \\
  & FedRAW (Both)       & $75.7 \pm 2.0$ & $43.2 \pm 17.7$$\uparrow$ & $49.9 \pm 14.0$$\uparrow$ & $\mathbf{59.0 \pm 2.0}$$\uparrow$ \\
\bottomrule
\end{tabular}
}
\end{table}

\begin{table}[t]
\centering
\caption{Staleness, Participation, and Rare-Class Analysis Across Speed
Ratios ($N=30$, $K=10$, Seed~42). Speed ratio = ratio of
rare to common client average update time. Expected rare participation = 20\% = (6 rare clients / 30 total).}
\label{tab:staleness}
\resizebox{\columnwidth}{!}{
\begin{tabular}{llcccccc}
\toprule
\textbf{Speed} &
\textbf{Method} &
\shortstack{\textbf{Global}\\\textbf{Acc}} & 
\shortstack{\textbf{Rare}\\\textbf{Part.}} & 
\shortstack{\textbf{Rare}\\\textbf{Mean}\\\textbf{Stale}} & 
\shortstack{\textbf{Max}\\\textbf{Stale}} & 
\shortstack{\textbf{Avg}\\\textbf{Rare}} &
\shortstack{\textbf{Buffer}\\\textbf{Presence}} \\
 
\midrule

\multirow{2}{*}{2$\times$}
& FedFA & 77.95\% & 10.34\% 
& 56.0  
& 336  
& 9.92\%  & 67.6\% \\
& FedRAW & 77.44\% & 9.44\%  
& 61.9  
& 340 
& \textbf{21.58\%}$\uparrow$ & \textbf{76.7\%}$\uparrow$ \\

\midrule

\multirow{2}{*}{10$\times$}
& FedFA & 77.25\% & 2.66\% 
& 214.6 
& 1064 
& 0.25\% & 23.8\% \\
& FedRAW & 74.35\% & 2.64\% 
& 218.0 
& 976  
& \textbf{4.75\%}$\uparrow$ & \textbf{33.8\%}$\uparrow$ \\

\midrule

\multirow{2}{*}{20$\times$}
& FedFA & 76.43\% & 1.20\% 
& 464.9 
& 2367 
& 0.00\% & 11.6\% \\
& FedRAW & 73.59\% & 1.34\% 
& 417.2 
& 2185 
& \textbf{0.83\%}$\uparrow$ & \textbf{17.9\%}$\uparrow$ \\

\bottomrule
\end{tabular}
}
\end{table}

\subsection{Ablation Study}
\label{subsec:ablation}

We examine which components of FedRAW are responsible for
rare-label recovery, how the method behaves under different buffer
sizes and rare-label densities, and whether its gains are specific
to label-coverage scarcity.

\subsubsection{Effect of Deduplication and Rarity Weighting}
\label{subsubsec:ablation_components}

We isolate the two mechanisms under the correlated-speed setting
($N=30$, $K=10$, 3 seeds). \textbf{FedRAW-DedupOnly} applies
deduplication with uniform $1/K$ weights;
\textbf{FedRAW-WeightOnly} applies rarity-weighted aggregation
without deduplication (Table~\ref{tab:component_ablation}).

\emph{Rarity weighting is the main driver of rare-label recovery; deduplication provides structural reliability}.
The two mechanisms contribute asymmetrically. Rarity-aware
weighting alone recovers 35.2\% AvgRare on EMNIST and 53.8\% on
HAM10000, accounting for the large majority of FedRAW's gains over
FedFa (6.0\% and 12.8\%), whereas deduplication alone yields only
7.9\% and 15.4\%. This reflects their different roles: weighting
directly amplifies rare-label influence at aggregation time, while
deduplication improves the structural reliability of rare-client
buffer presence (BufPres rises from 59.9\% to 72.7\% on EMNIST and
46.6\% to 59.0\% on HAM10000), reducing dependence on stochastic
buffer events and making the gains from weighting more consistent
across seeds. Full FedRAW combines both: weighting without
deduplication is effective on average but more seed-dependent,
deduplication without weighting is structurally useful but
insufficient alone, together they give reliable server-side
correction.

\subsubsection{Effect of Buffer Size}
\label{subsubsec:buffer_size}

We next sweep the buffer size $K \in \{3,5,7,10,11,20\}$ on EMNIST
Balanced under correlated speed (Supplementary Section~S-VII).
FedFa's rare-label performance generally degrades as the buffer
grows, with AvgRare dropping from 20.92\% at $K=3$ to 14.08\% at
$K=10$ and 13.92\% at $K=20$ -- consistent with the proposed
failure mechanism, since larger buffers give frequently arriving
common clients more opportunity to occupy positions and dilute
rare-client presence under uniform aggregation. FedRAW maintains
substantial gains across all tested sizes (28.42--42.17\%
AvgRare), indicating that both mechanisms remain effective beyond a
single choice of $K$.

\subsubsection{Effect of Rare-Label Density}
\label{subsubsec:rare_density}

Varying the number of clients holding each rare label from 1 to 10
on EMNIST Balanced (Supplementary Section~S-VII), AvgRare
increases from 18.25\% at one client per rare label to 52.75\% at
ten, following the rise in rare-client buffer presence, which
reaches 100\% by seven clients. The improvement is balanced rather
than precision-driven: RareF1 rises consistently with AvgRare
while MacroF1 stays within 1.5\,pp of GlobalAcc, indicating that
broader rare-label coverage improves rare-class recovery without
degrading common-class quality.

\subsubsection{Specificity to Label-Coverage Scarcity}
\label{subsubsec:coverage_specificity}

FedRAW is designed for label-coverage rarity, not generic non-IID
imbalance. To verify this, we evaluate a Dirichlet partition with
$\alpha=0.1$~\cite{hsu2019}. Under this split all clients hold all
labels, so label coverage becomes approximately uniform and
FedRAW's rarity scores collapse toward uniform weighting. In this
setting FedRAW shows no meaningful improvement over FedFa,
confirming that its benefit arises specifically when labels are
held by only a small number of clients. Full results are in
Supplementary Section~S-IV.

\subsubsection{Robustness to Rare-Label Choice and Scale}
\label{subsubsec:robustness_scale}

Changing the rare labels from {44,45,46} to {20,30,40} on EMNIST
Balanced, FedFa collapses to 4.0\% AvgRare while FedRAW recovers
17.3\% ($+13.3$\,pp), showing that the failure is not tied to
specific label identities but is a structural property of
asynchronous buffer aggregation under label-coverage scarcity.
Increasing the number of clients to $N=50$, FedRAW maintains a
$+15.2$\,pp AvgRare gain over FedFa (18.6\% vs.\ 3.4\%),
confirming effectiveness beyond the main $N=30$ setting (Supplementary Sections~S-II and~S-I).

\subsubsection{Robustness to Rarity Misreporting}

FedRAW's rarity score $S_i$ relies on client-reported label
statistics (Sec.~III-B) that the server does not verify. To probe
this, we designate one common-label client as an adversary
falsely claiming 90\% of its data belongs to rare label 44, while
leaving its actual data and training unchanged. Without defense
the attack costs 6.17\,pp AvgRare and 2.31\,pp GlobalAcc. A simple
clip-and-renormalize cap on $w_i$ does not enforce the intended
bound, since renormalization re-inflates the capped weight; we
instead pin any client exceeding $w_i > 0.3$ at exactly 0.3 and
redistribute only the freed mass across the remaining uncapped
clients, iterating until the bound holds everywhere. This
water-filling cap recovers 83.5\% of the GlobalAcc loss and 42.0\%
of the AvgRare loss. The residual gap stems from the attacker's
buffer occupancy rather than its per-event weight; the full
comparison, including the naive cap and occupancy statistics, is
in Supplementary Section~S-VI. A complete defense likely requires
combining weight capping with occupancy-aware safeguards, which we
leave to future work.

\subsection{Effect of Staleness under Extreme Speed Heterogeneity}
\label{subsec:staleness}

We examine whether staleness or participation collapse is the dominant failure mode under extreme system heterogeneity, extending the main EMNIST Balanced setting ($N=30$, $K=10$) by raising the rare/common average speed ratio from 2$\times$ to 10$\times$ and 20$\times$ with common-client speeds fixed at $[0.5,1.5]$ seconds. Table~\ref{tab:staleness} reports rare-client participation, staleness, and AvgRare for FedFa and FedRAW.

\paragraph{Rare-client participation collapses as speed
heterogeneity increases}
Rare-client participation drops from 9.44\% at 2$\times$ to 2.64\%
at 10$\times$ and 1.34\% at 20$\times$ under FedRAW, far below the
proportional share of 20\% (6 rare clients of 30), with a similar
drop under FedFa. At 20$\times$ each rare client submits only
about 11 updates over 5,000 events, leaving very little rare-label
signal for any aggregation method. Deduplication mitigates this by
improving rare-client buffer presence at every speed ratio
(\textbf{76.7}\% vs.\ 67.6\% at 2$\times$; \textbf{33.8}\% vs.\
23.8\% at 10$\times$; \textbf{17.9}\% vs.\ 11.6\% at 20$\times$),
preserving rare-client slots even as participation falls.

\paragraph{Rare-client staleness grows, but it is not the
main cause of collapse}
Rare-client mean staleness grows from 61.9 rounds at 2$\times$ to
218.0 at 10$\times$ and 417.2 at 20$\times$ under FedRAW, yet
GlobalAcc decreases gradually alongside it (77.44\%, 74.35\%, and
73.59\% respectively), a bounded $\sim$3.85\,pp drop across a
10$\times$ change in speed ratio -- indicating that buffered
aggregation dilutes the effect of individual stale updates.

\paragraph{Participation bias is the dominant bottleneck at
extreme heterogeneity}
At 20$\times$, FedFa's AvgRare collapses to 0.00\% while FedRAW
retains 0.83\%, itself a near-floor value, but FedRAW's relative
advantage widens sharply with speed ratio, from 2.2$\times$ at
2$\times$ (21.58\% vs.\ 9.92\%) to 19$\times$ at 10$\times$
(4.75\% vs.\ 0.25\%). FedRAW therefore more effectively exploits the rare-label signal that reaches the buffer even as that signal
itself vanishes, which identifies delay-aware or rarity-aware
client scheduling as a complementary direction for extreme
heterogeneity.

\section{Limitation}
\label{sec:limitation}
FedRAW amplifies updates from rare-label clients based on label
coverage, without accounting for feature-space relationships between
rare classes. This assumption may be less effective when rare classes
share strong visual features. We observe this boundary on CIFAR-10 with
Cat and Dog as rare labels: despite FedRAW maintaining the highest
rare-client buffer presence across all seeds, its mean rare-label
accuracy falls below competing methods. The failure is one not of
participation but of amplification instability: rarity weighting
amplifies both Cat and Dog clients simultaneously, intensifying their
competition during early training, so recovery depends on which rare
client first dominates the buffer (Supplementary Section~S-V). When
rare classes share strong visual features, participation-guarantee
mechanisms such as CA2FL, which use uniform aggregation weights, provide
more consistent recovery. Combining rarity-aware weighting with
inter-class similarity awareness is a direction for future work.

A second limitation concerns the one-time label summary collected
at initialization: it keeps the mechanism lightweight, but assumes
each client's label distribution is known upfront, making initialization
synchronous and harder to reconcile with dynamic membership or with
label distributions that drift during training. FedRAW can therefore
compensate for labels that are rare but present among participating
clients, not for labels absent at initialization or introduced later.
Incremental or re-estimated rarity scoring under dynamic participation
is a promising direction for future work.

\section{Conclusion}
We identified \emph{silent rarity failure}, a structural failure
mode in buffer-based asynchronous federated learning where
rare-label clients lose aggregation influence while global
accuracy remains largely unchanged. It arises from participation
bias, where rare-label clients enter the update buffer less
frequently, and aggregation bias, where uniform buffer aggregation
gives their updates no compensating influence once present. This
is a \emph{system-induced} imbalance, driven by asynchronous
buffer dynamics rather than by the global sample distribution, and
it can occur even when rare-label clients hold sufficient local
data -- making it distinct from, and complementary to, classical
class-imbalance and long-tail learning problems.

We proposed FedRAW, a fully server-side aggregation method
combining client-level buffer deduplication with rare-label-aware
weighting, which improves rare-label influence without modifying
local training, client objectives, or communication protocols,
requiring only a one-time client label summary at initialization.
Across EMNIST Balanced, CIFAR-10, HAM10000, and ISIC-2019, FedRAW
improves AvgRare over FedFa by $+24.2$\,pp, $+11.8$\,pp,
$+30.4$\,pp, and $+18.0$\,pp respectively at comparable global
accuracy. 
Compared with CA2FL, FedRAW achieves competitive rare-label performance across datasets, outperforming it when aggregation influence is the dominant bottleneck, while requiring only $O(K)$ rather than $O(N)$ server-side state.
F1-based and client-level metrics confirm genuine
rare-label recovery, and ablations show that rarity-aware
weighting drives it while deduplication improves structural
reliability by increasing rare-client buffer presence.

Future work will address the synchronous label-summary
initialization (Section~\ref{sec:limitation}) through incremental
or re-estimated rarity scoring, combine FedRAW with delay-aware or
rarity-aware client scheduling, and explore privacy-preserving
computation of label-coverage statistics.

 

\clearpage
\setcounter{section}{0}
\section*{Supplementary Material}


\newtheorem{corollary}{Corollary}
\newtheorem{assumption}{Assumption}
\newtheorem{remark}{Remark}


\renewcommand{\thesection}{S-\Roman{section}}
\renewcommand{\thesubsection}{S-\Roman{section}.\Alph{subsection}}
\renewcommand{\thetable}{S\arabic{table}}
\renewcommand{\thefigure}{S\arabic{figure}}

\textbf{Additional experiments:
(\textbf{S-I})~scalability to $N{=}50$ clients,
(\textbf{S-II})~generalisation to alternative rare-label sets,
(\textbf{S-III})~q-FedAvg sensitivity to the hyperparameter $q$,
(\textbf{S-IV})~FedRAW behaviour under Dirichlet non-IID
partitioning,
(\textbf{S-V})~CIFAR-10 results with visually similar rare
classes (cat and dog),
(\textbf{S-VI})~naive versus water-filling weight capping under
rarity misreporting,
(\textbf{S-VII})~buffer-size and rare-label-density ablations, and
(\textbf{S-VIII})~the full convergence analysis, containing the
complete proof of Theorem~1 of the main paper.
All experimental settings match those described in Section~V-A of the main paper unless otherwise noted.}

\section{Scalability: $N = 50$}
\label{supp:scalability}

Table~\ref{tab:supp_n50} extends the main EMNIST Balanced
experiment to $N{=}50$ clients ($K{=}10$, 3~seeds: 42, 123,
456), with 6 rare-label clients (2 per rare-label) and 44
common-label clients.
FedRAW maintains a $+15.2$\,pp AvgRare advantage over FedFa
($18.6\%$ vs.\ $3.4\%$) with GlobalAcc comparable to FedFa,
confirming that the method scales beyond the primary $N{=}30$
setting. The higher variance at $N{=}50$ ($\pm5.5$\,pp)
reflects the proportionally smaller rare-client pool: with only
6 rare clients out of 50, each participation event has a greater
impact on the seed-averaged outcome.

\begin{table}[ht]
\centering
\caption{Scalability results: $N{=}50$, $K{=}10$, 3~seeds
         (42, 123, 456). Best async result \textbf{bold}.}
\label{tab:supp_n50}
\small
\begin{tabular}{lcc}
\toprule
\textbf{Algorithm} & \textbf{GlobalAcc (\%)} & \textbf{AvgRare (\%)} \\
\midrule
FedAvg (sync)          & $76.35\pm0.30$ & $23.33\pm1.18$ \\
\midrule
FedBuff (async)        & $72.20\pm0.25$ & $3.69\pm1.77$  \\
FedFa   (async)        & $72.58\pm0.23$ & $3.44\pm0.94$  \\
\textbf{FedRAW (async)}& $\mathbf{70.12\pm1.31}$ &
                         $\mathbf{18.61\pm5.49}$ \\
\bottomrule
\end{tabular}
\end{table}

\section{Generalisation to Alternative Rare-Label Sets}
\label{supp:generalization}

To verify that silent rarity failure and FedRAW's recovery are
not artefacts of the specific label indices \{44, 45, 46\} used
in the main EMNIST experiments, we repeat the correlated-speed
evaluation with a different rare-label set \{20, 30, 40\} under
otherwise identical conditions ($N{=}30$, $K{=}10$, seed~42).

Table~\ref{tab:supp_generalization} shows that the failure
pattern is structural and label-agnostic: FedFa collapses to
$4.0\%$ AvgRare regardless of which labels are designated rare,
and FedRAW recovers $17.3\%$ ($+13.3$\,pp), confirming that
participation and aggregation biases arise from buffer dynamics
rather than any property specific to a particular label set.

\begin{table}[ht]
\centering
\caption{Generalisation to rare labels \{20, 30, 40\},
         EMNIST Balanced, correlated speed, $N{=}30$, $K{=}10$,
         seed~42. Best async result \textbf{bold}.}
\label{tab:supp_generalization}
\small
\resizebox{\columnwidth}{!}{%
\begin{tabular}{lccc}
\toprule
\textbf{Algorithm} & \textbf{GlobalAcc(\%)} &
\textbf{AvgRare(\%)} & \textbf{Worst-10\%} \\
\midrule
FedAvg (sync)           & 83.06 & 42.33 & 69.87 \\
\midrule
FedBuff (async)         & 76.78 & 14.08 & 58.80 \\
FedFa   (async)         & 78.27 &  4.00 & 57.62 \\
\textbf{FedRAW (async)} & \textbf{78.44} &
                          \textbf{17.33} &
                          \textbf{61.52} \\
\bottomrule
\end{tabular}
}
\end{table}

\section{q-FedAvg Sensitivity to Hyperparameter $q$}
\label{supp:qfedavg}

The main paper reports q-FedAvg results at $q{=}5$ across three
seeds. Table~\ref{tab:supp_qfedavg} reports a single-seed
($seed{=}42$) sweep across $q \in \{0.1, 1.0, 5.0\}$ on EMNIST
Balanced, with FedAvg and FedRAW included as reference points.

Three observations follow. First, at $q{=}0.1$ the method
reduces to near-FedAvg behaviour (GlobalAcc $81.3\%$, AvgRare
$32.8\%$), confirming that meaningful loss-based reweighting
requires $q \geq 1$. Second, at $q{=}1.0$ global accuracy drops
to $70.0\%$ with AvgRare $16.1\%$ -- worse on both metrics than
FedRAW. Third, at $q{=}5.0$ the method degenerates to near-zero
AvgRare ($0.1\%$) and $50.3\%$ GlobalAcc, confirming that
device-level loss reweighting is not well-suited to the
label-coverage rarity problem at any tested $q$ value. Three-seed
results for $q{=}5$ are reported in Table~II of the main paper;
results for $q{=}0.1$ and $q{=}1.0$ are not repeated across seeds,
as the single-seed sweep is sufficient to characterise the
sensitivity.

\begin{table}[ht]
\centering
\caption{q-FedAvg sensitivity to $q$, EMNIST Balanced,
         Same data split, $N{=}30$, $K{=}10$, seed~42.
         FedAvg and FedRAW shown as reference.}
\label{tab:supp_qfedavg}
\small
\begin{tabular}{lcccc}
\toprule
\textbf{Method} & \textbf{GlobalAcc} & \textbf{AvgRare} &
\textbf{RareF1} & \textbf{RareF2} \\
\midrule
FedAvg (sync)        & 81.7 & 38.8 & 50.4 & 42.8 \\
\midrule
q-FedAvg $q{=}0.1$  & 81.3 & 32.8 & 45.1 & 36.7 \\
q-FedAvg $q{=}1.0$  & 70.0 & 16.1 & 25.7 & 18.9 \\
q-FedAvg $q{=}5.0$  & 50.3 &  0.1 &  0.1 &  0.1 \\
\midrule
FedFa  (async)       & 76.5 & 11.6 & 18.9 & 13.9 \\
FedRAW (async)       & 75.5 & 24.7 & 32.6 & 27.8 \\
\bottomrule
\end{tabular}
\end{table}

\section{FedRAW Under Dirichlet Non-IID Partitioning}
\label{supp:dirichlet}

\subsection{Setup and Motivation}

To verify that FedRAW does not introduce accuracy or fairness
regression under general non-IID conditions, we evaluate on
Dirichlet-partitioned EMNIST Balanced with concentration
parameters $\alpha \in \{0.1, 0.5, 5.0\}$,
spanning highly skewed to near-uniform label distributions.

\subsection{Why FedRAW is Inert Under Dirichlet Partitioning}

Under Dirichlet partitioning every client receives samples from
all labels, so label coverage is approximately uniform:
$\text{coverage}(l) \approx N$ for all $l$.
Consequently, all rarity scores $S_i$ collapse to near-equal
values and the rarity-weighted aggregation reduces to
approximately uniform $1/K$ averaging, recovering standard
FedFa behaviour while retaining FedRAW's structural
deduplication constraint.
The condition FedRAW is designed to address — structural
label-coverage scarcity — does not exist in this setting,
so no compensation occurs and no degradation is expected.

\subsection{Results}

Table~\ref{tab:supp_dirichlet} reports GlobalAcc, MeanClient,
Worst-10\%, and Jain's Fairness Index across all three $\alpha$
values. FedRAW and FedFa converge at nearly identical rates
and reach the same final accuracy on all metrics, with
differences within noise. AvgRare is omitted because under
Dirichlet partitioning rare labels are not structurally defined
— the bottom-$k$ labels by sample frequency vary across seeds
and $\alpha$ values, making cross-run comparison of per-class
rare accuracy meaningless.

\subsection{Interpretation}

These results confirm two properties of FedRAW. First, the
method does not introduce accuracy or fairness degradation under general non-IID conditions: it
does not introduce any accuracy or fairness regression relative
to FedFa when coverage scarcity is absent. Second, the gains
reported in the long-tail experiments of the main paper are not
an artefact of the partitioning strategy — they arise
specifically from the structural label-coverage scarcity that
the long-tail split induces, which is precisely the condition
FedRAW is designed to address.

\begin{table}[t]
\centering
\caption{Impact of Dirichlet heterogeneity ($\alpha$) on model
         performance and fairness ($N{=}30$, $K{=}10$, seed~42).
         Higher values are better in all columns.}
\label{tab:supp_dirichlet}
\resizebox{\columnwidth}{!}{%
\begin{tabular}{clcccc}
\toprule
$\alpha$ & \textbf{Algorithm} &
\textbf{GlobalAcc (\%)} &
\textbf{MeanClient (\%)} &
\textbf{Worst-10\% (\%)} &
\textbf{Jain's Index} \\
\midrule
\multirow{4}{*}{0.1}
& FedAvg  & \textbf{81.71} & \textbf{82.88} & \textbf{68.94} & \textbf{0.9936} \\
& FedBuff & 76.48 & 77.65 & 60.87 & 0.9890 \\
& FedFa   & 77.39 & 79.36 & 65.97 & 0.9923 \\
& FedRAW  & 77.05 & 78.51 & 66.21 & 0.9917 \\
\midrule
\multirow{4}{*}{0.5}
& FedAvg  & \textbf{84.03} & \textbf{85.59} & \textbf{77.72} & \textbf{0.9984} \\
& FedBuff & 81.66 & 82.79 & 73.77 & 0.9975 \\
& FedFa   & 82.34 & 83.50 & 72.50 & 0.9972 \\
& FedRAW  & 81.93 & 83.05 & 72.17 & 0.9971 \\
\midrule
\multirow{4}{*}{5.0}
& FedAvg  & \textbf{85.10} & \textbf{86.82} & \textbf{84.82} & \textbf{0.9998} \\
& FedBuff & 83.70 & 85.49 & 82.91 & 0.9997 \\
& FedFa   & 84.28 & 85.98 & 83.48 & 0.9997 \\
& FedRAW  & 83.54 & 85.01 & 82.56 & \textbf{0.9998} \\
\bottomrule
\end{tabular}
}
\end{table}

\section{CIFAR-10: Visually Similar Rare Labels (Cat \& Dog)}
\label{supp:cifar_catdog}

\subsection{Setup}

To examine FedRAW's behaviour when rare classes share visual
features, we evaluate CIFAR-10 with Cat~(label~3) and Dog
(label~5) designated as rare, using the same correlated-speed
setting as the main paper ($N{=}30$, $K{=}10$, seeds~42,
123,~456). This supplements the primary CIFAR-10 results
(airplane and frog as rare labels, Table~II of the main paper)
and is discussed in Section~VI of the main paper.

\subsection{Main Results}

Table~\ref{tab:supp_catdog_main} reports accuracy and F1
metrics across three seeds.
FedRAW's mean AvgRare ($14.4\% \pm 14.3$) falls below both
CA2FL ($26.5\% \pm 1.4$) and FedFa ($18.5\% \pm 1.5$),
with dramatically higher variance reflecting seed-dependent
winner-takes-all dynamics between the two visually similar
rare classes. This contrasts with the primary CIFAR-10 setting
(airplane and frog) where FedRAW leads all async baselines
by $+11.8$\,pp AvgRare.

\begin{table*}[ht]
\centering
\caption{CIFAR-10, Cat~(3) and Dog~(5) as rare labels,
         correlated speed, $K{=}10$, $N{=}30$, 3~seeds
         (42, 123, 456). High FedRAW variance reflects
         seed-dependent winner-takes-all dynamics between
         visually similar rare classes. Best async result
         \textbf{bold}.}
\label{tab:supp_catdog_main}
\resizebox{\textwidth}{!}{%
\begin{tabular}{lcccccc}
\toprule
\textbf{Algorithm} &
\textbf{GlobalAcc (\%)} &
\textbf{AvgRare (\%)} &
\textbf{MacroF1 (\%)} &
\textbf{RareF1 (\%)} &
\textbf{Local Rare (\%)} &
\textbf{Local Common (\%)} \\
\midrule
FedAvg (sync)   & $58.5\pm0.6$ & $39.2\pm1.4$ & $58.4\pm0.3$
                & $35.3\pm2.4$ & $46.5\pm1.1$ & $60.7\pm1.1$ \\
q-FedAvg (sync) & $10.0\pm0.0$ & $0.0\pm0.0$  & $1.8\pm0.0$
                & $0.0\pm0.0$  & $5.9\pm1.2$  & $11.9\pm0.7$ \\
\midrule
\textbf{CA2FL (async)} & $\mathbf{56.3\pm1.2}$ & $\mathbf{26.5\pm1.4}$
                & $\mathbf{55.4\pm1.1}$ & $\mathbf{29.0\pm1.5}$
                & $\mathbf{39.0\pm1.4}$ & $\mathbf{61.7\pm1.0}$ \\
FedBuff (async) & $57.5\pm0.4$ & $17.8\pm4.3$ & $55.0\pm0.5$
                & $20.0\pm4.2$ & $35.0\pm2.7$ & $65.3\pm0.7$ \\
FedFa   (async) & $54.3\pm0.3$ & $18.5\pm1.5$ & $52.6\pm0.6$
                & $22.5\pm2.8$ & $34.7\pm2.1$ & $61.1\pm0.5$ \\
FedRAW  (async) & $54.6\pm0.4$ & $14.4\pm14.3$ & $51.1\pm1.5$
                & $12.0\pm7.5$ & $35.5\pm9.2$  & $61.1\pm4.8$ \\
\bottomrule
\end{tabular}
}
\end{table*}

\subsection{Per-Seed Breakdown and Buffer Presence}

Table~\ref{tab:supp_catdog_detail} shows per-seed AvgRare and
rare-client buffer presence (BufPres) for FedFa, CA2FL, and
FedRAW. Notably, FedRAW achieves the \emph{highest} BufPres
of all async methods on every seed ($56.7\%$, $53.4\%$,
$58.2\%$), confirming that the failure is not caused by reduced
rare-client participation. On seed~42, FedRAW achieves
$34.4\%$ AvgRare (highest of all methods including FedAvg),
but collapses to $1.4\%$ and $7.5\%$ on seeds~123 and~456.
This seed-dependence is consistent with a winner-takes-all failure mode: rarity-weighted amplification intensifies
gradient conflict between two classes sharing visual features,
and the outcome is determined by which rare client happened
to dominate the buffer during early training.

\begin{table}[ht]
\centering
\small
\caption{Per-seed AvgRare and rare-client buffer presence
         (BufPres) for the Cat+Dog setting. FedRAW achieves
         the highest BufPres on every seed, confirming the
         failure is not caused by reduced participation.}
\label{tab:supp_catdog_detail}
\renewcommand{\arraystretch}{1.2}
\begin{tabular}{lcccccc}
\toprule
& \multicolumn{3}{c}{\textbf{AvgRare (\%)}} &
  \multicolumn{3}{c}{\textbf{BufPres (\%)}} \\
\cmidrule(lr){2-4}\cmidrule(lr){5-7}
\textbf{Algorithm} &
\textbf{S42} & \textbf{S123} & \textbf{S456} &
\textbf{S42} & \textbf{S123} & \textbf{S456} \\
\midrule
FedFa  & 16.7 & 18.4 & 20.4 & 45.2 & 45.0 & 45.5 \\
CA2FL  & 24.6 & 28.0 & 26.9 & 47.4 & 41.4 & 45.6 \\
FedRAW & 34.4 &  1.4 &  7.5 & 56.7 & 53.4 & 58.2 \\
\bottomrule
\end{tabular}
\end{table}

\subsection{Confusion Matrix Analysis}

Table~\ref{tab:supp_catdog_cm} shows the mean confusion matrix
rows for Cat and Dog, averaged across three seeds.
Under FedFa, dog samples are predominantly predicted as dog
(440 correct vs.\ 97 misclassified as cat), reflecting
reasonable class separation under uniform weighting.
Under FedRAW, the dominant error reverses: 356 of 1000 dog
test samples are predicted as cat on average across seeds
(698 on seed~42 alone), while cat recall remains comparable.
This asymmetric collapse confirms the gradient interference
hypothesis: FedRAW's ${\sim}4\times$ rarity weighting
amplifies gradient conflict between visually overlapping
classes, producing a seed-dependent winner in which one
rare class erases the other rather than both recovering.

\begin{table}[ht]
\centering
\small
\caption{Mean confusion matrix rows for Cat~(3) and Dog~(5),
         averaged across 3~seeds. Classes:
         0=airplane, 1=auto, 2=bird, 3=cat, 4=deer,
         5=dog, 6=frog, 7=horse, 8=ship, 9=truck.
         Bold entries indicate the predicted class receiving
         the most samples for each true class.}
\label{tab:supp_catdog_cm}
\renewcommand{\arraystretch}{1.2}
\resizebox{\columnwidth}{!}{%
\begin{tabular}{llcccccccccc}
\toprule
& \textbf{True} &
\textbf{0} & \textbf{1} & \textbf{2} & \textbf{3} &
\textbf{4} & \textbf{5} & \textbf{6} & \textbf{7} &
\textbf{8} & \textbf{9} \\
\midrule
\multirow{2}{*}{FedFa}
& Cat & 14  & 20 & 118 & \textbf{221} & 85  & 97          & 93  & 79  & 29 & 17 \\
& Dog & 31  & 12 & 120 & 97           & 87  & \textbf{440}& 54  & 106 & 25 & 14 \\
\midrule
\multirow{2}{*}{FedRAW}
& Cat & 33  & 20 & 24  & \textbf{222} & 41  & 6           & 125 & 42  & 32 & 23 \\
& Dog & 47  & 16 & 201 & \textbf{356} & 140 & 33          & 66  & 181 & 30 & 15 \\
\bottomrule
\end{tabular}
}
\end{table}

\section{Rarity Misreporting: Naive vs Water-Filling weight capping}
\label{sec:rarity_misreporting}

\subsection{Setup}

Section~V-C6 of the main paper introduces an adversarial client that falsely reports
90\% of its data as belonging to Label 44 while its actual local data
and training remain unchanged, and shows that a water-filling weight
cap recovers most of the resulting accuracy loss. Here we report the
full comparison, including the naive clip-and-renormalize cap that
motivated the water-filling design, and the mechanism behind the
residual gap, under identical settings (EMNIST Balanced, correlated
speed, $N{=}30$, $K{=}10$, seed 42).

\subsection{Why the Naive Cap Fails}

Given normalized weights $w_i$ and a cap $c$, the naive approach
clips every weight exceeding $c$ and renormalizes the entire buffer
by the same factor:
\[
w_i' = \frac{\min(w_i, c)}{\sum_{j \in B} \min(w_j, c)}.
\]
Because the denominator divides \emph{every} weight -- including
ones already clipped to $c$ -- by a factor less than 1 whenever any
clipping occurs, a capped client's final weight $w_i'$ can exceed the
nominal cap $c$. In our attack, the adversary's pre-cap weight
averages 0.466; clipping to $c{=}0.3$ and uniformly renormalizing
raises its realized weight back to 0.317--0.388 (mean 0.361) --
above the intended ceiling.

The water-filling cap instead pins any client exceeding $c$ at
exactly $c$, and redistributes only the freed probability mass across
clients that remain below $c$, repeating until no weight in the
buffer exceeds $c$. This guarantees $w_i' \le c$ for every buffered
client, at the cost of a small iterative loop (at most $K$ rounds per
aggregation).

\subsection{Results}

Table~\ref{tab:naive_vs_waterfill} reports the full accuracy
comparison across four conditions -- no attack, attack with no
defense, attack with the naive cap, and attack with the water-filling
cap -- together with the attacker's realized weight and each
condition's buffer-occupancy statistics, computed from the
per-aggregation weight log across all 4{,}989 buffer-full events.

The naive cap recovers only 13.5\% of the AvgRare loss and 56\% of
the GlobalAcc loss, consistent with its failure to enforce the bound.
The water-filling cap recovers 42.0\% of the AvgRare loss and 83.5\%
of the GlobalAcc loss -- roughly $3\times$ the naive cap's recovery on
every accuracy metric reported.

The occupancy columns explain why even a perfectly enforced weight
cap leaves a residual AvgRare gap. Weight capping bounds the
\emph{magnitude} of a client's realized weight but does not affect
\emph{whether} it occupies the buffer: the attacker holds the
maximum-weight slot in 53.8\% of aggregations under both the
undefended and water-filling-capped conditions -- identical in both
cases -- while genuine rare clients hold the maximum slot in only
37.2\% of aggregations once the attacker is present, down from 76.7\%
in the unattacked run.

\begin{table}[ht]
\centering
\caption{Naive vs.\ Water-Filling Weight Cap under Rarity
Misreporting. EMNIST Balanced, Correlated Speed, $N{=}30$, $K{=}10$,
Seed 42. Attacker weight reports the client's realized aggregation
weight whenever it holds the buffer maximum. Occupancy columns report
the frequency, across all 4{,}989 aggregation events, with which the
buffer maximum is held by a genuine rare client vs.\ the attacker;
capping affects weight magnitude only, not buffer occupancy.}
\label{tab:naive_vs_waterfill}
\resizebox{\columnwidth}{!}{%
\begin{tabular}{lccccc}
\toprule
Condition & GlobalAcc(\%) & AvgRare(\%) & Attacker $w_i$ & Rare holds max & Attacker holds max \\
\midrule
No attack                          & 77.96 & 23.75 & --            & 76.7\% & -- \\
Attack, no defense                 & 75.65 & 17.58 & 0.353--0.526  & 37.2\% & 53.8\% \\
Attack, naive cap ($c{=}0.3$)      & 76.95 & 18.42 & 0.317--0.388  & 37.2\% & 53.8\% \\
Attack, water-filling ($c{=}0.3$)  & \textbf{77.58} & \textbf{20.17} & \textbf{0.300} & 37.2\% & 53.8\% \\
\bottomrule
\end{tabular}
}
\end{table}

\subsection{Interpretation}

Two findings follow. First, a weight cap must be designed with its
interaction with renormalization in mind -- a naive clip is
insufficient because renormalization can restore the very excess it
was meant to remove. Second, even a perfectly enforced cap leaves a
residual gap: the attacker holds the buffer's maximum-weight slot in
53.8\% of aggregations regardless of capping, since weight capping
does not alter which clients occupy the buffer. Fully closing this
gap likely requires an occupancy-aware safeguard -- e.g., flagging
clients whose reported rarity is inconsistent with their submitted
updates -- rather than weight capping alone.

\section{Additional Ablations: Buffer Size and Rare-Label Density}
\label{supp:additional}

This appendix reports the two ablations referenced from
Section V-C of the main paper but omitted there for space: the
buffer-size sweep (Section V.C.2) and the rare-label-density
sweep (Section V.C.3).

\subsection{Effect of Buffer Size}
\label{supp:buffer}

Table~\ref{tab:supp_buffer} reports GlobalAcc, AvgRare,
MeanClient, and Worst-10\% for FedFa and FedRAW across
$K \in \{3,5,7,10,11,20\}$ on EMNIST Balanced under correlated
speed, supporting Section~V-C2 of the main paper. FedFa's
rare-label performance generally degrades as the buffer grows,
from 20.92\% AvgRare at $K=3$ to 14.08\% at $K=10$ and 13.92\% at
$K=20$, consistent with the proposed failure mechanism: larger
buffers give frequently arriving common clients more opportunity
to occupy positions, diluting rare-client presence under uniform
aggregation. FedRAW retains substantial rare-label gains at every
tested buffer size.

\begin{table}[ht]
\centering
\caption{Buffer Size Ablation: EMNIST Balanced, Correlated Speed,
$N$=30, Seed~42.}
\label{tab:supp_buffer}
\resizebox{\columnwidth}{!}{
\begin{tabular}{cccccc}
\toprule
\textbf{K} & \textbf{Method}
& \shortstack{\textbf{GlobalAcc}\textbf{(\%)}} 
& \shortstack{\textbf{AvgRare}\textbf{(\%)}} 
& \shortstack{\textbf{MeanClient}\textbf{(\%)}}
& \shortstack{\textbf{Worst-10}\textbf{(\%)}} \\
\midrule
3  & FedFa  & \textbf{76.97} & 20.92 & \textbf{79.07} & \textbf{60.36} \\
   & FedRAW & 68.86 & \textbf{35.08}$\uparrow$ & 70.59 & 55.90 \\
\midrule
5  & FedFa  & \textbf{79.71} & 20.75 & \textbf{82.18} & \textbf{61.80} \\
   & FedRAW & 74.02 & \textbf{30.50}$\uparrow$ & 75.77 & 61.68 \\
\midrule
7  & FedFa  & \textbf{79.43} & 16.25 & \textbf{82.04} & 60.61 \\
   & FedRAW & 76.13 & \textbf{31.17}$\uparrow$ & 77.95 & \textbf{64.20} \\
\midrule
10 & FedFa  & 78.98 & 14.08 & 81.03 & 58.56 \\
   & FedRAW & \textbf{79.06} & \textbf{40.58}$\uparrow$ & \textbf{81.15} & \textbf{65.70} \\
\midrule
11 & FedFa  & \textbf{78.50} & 17.33 & \textbf{81.00} & 60.48 \\
   & FedRAW & 77.44 & \textbf{28.42}$\uparrow$ & 79.45 & \textbf{65.25} \\
\midrule
20 & FedFa  & \textbf{79.16} & 13.92 & \textbf{81.51} & 59.84 \\
   & FedRAW & 78.37 & \textbf{42.17}$\uparrow$ & 80.00 & \textbf{69.30} \\
\bottomrule
\end{tabular}
}
\end{table}

\subsection{Effect of Rare-Label Density}
\label{supp:density}

Table~\ref{tab:supp_density} varies the number of clients holding
each rare label from 1 to 10 on EMNIST Balanced under correlated
speed, supporting Section~V-C3 of the main paper. AvgRare rises
from 18.25\% at one client per rare label to 52.75\% at ten,
tracking the increase in rare-client buffer presence, which
reaches 100\% by seven clients per rare label. RareF1 rises
consistently with AvgRare while MacroF1 stays within 1.5\,pp of
GlobalAcc, indicating that broader rare-label coverage improves
rare-class recovery without degrading common-class quality.

\begin{table}[ht]
\centering
\caption{Effect of Varying no. of clients per rare label.
         EMNIST Balanced, correlated speed, $K{=}10$, seed~42.}
\label{tab:supp_density}
\small
\resizebox{\columnwidth}{!}{%
\begin{tabular}{cccccc}
\toprule
\textbf{Setting} & \textbf{GlobalAcc(\%)} & \textbf{AvgRare(\%)} & \textbf{MacroF1(\%)} & \textbf{RareF1(\%)} & \textbf{BufPres(\%)} \\
\midrule
1  & 76.85 & 18.25 & 75.45 & 28.55 &  45.05 \\
2  & 78.22 & 24.67 & 76.92 & 35.23 &  76.73 \\
3  & 77.36 & 22.83 & 75.82 & 33.81 &  90.33 \\
5  & 78.49 & 37.50 & 77.23 & 47.69 &  99.48 \\
7  & 79.77 & 48.92 & 79.08 & 59.28 & \textbf{100.00} \\
10 & \textbf{79.85} & \textbf{52.75} & \textbf{79.19} 
& \textbf{62.28} & \textbf{100.00} \\
\bottomrule
\end{tabular}
}
\end{table}

\subsection{Seed-wise Analysis of HAM10000 Rare-Label Metrics}

The relatively high variance in the HAM10000 rare-label metrics is due to the
seed-dependent client data partition and the randomness in client update
arrival times. Table~\ref{tab:ham_seedwise} reports the individual results for seeds 42, 123, and
456. Despite this variability, FedRAW consistently outperforms FedFa and
FedBuff across all three seeds, while remaining competitive with CA2FL, which
uses cached client updates to retain the influence of clients that participate
less frequently.

\begin{table}[t]
\centering
\caption{Seed-wise rare-label performance on HAM10000 for seeds 42, 123, and 456.}
\label{tab:ham_seedwise}
\resizebox{\columnwidth}{!}{%
\begin{tabular}{lccc}
\toprule
Algorithm & AvgRare(S-42/123/456) & RareF1(S-42/123/456) & RareF2(S-42/S123/S456) \\
\midrule
FedAvg
& 61.11 / 58.90 / 74.48
& 63.20 / 64.01 / 71.86
& 61.80 / 60.57 / 73.34 \\

q-FedAvg
& 14.81 / 19.02 / 27.08
& 18.38 / 27.32 / 35.71
& 15.98 / 21.63 / 29.84 \\

\midrule

CA2FL
& 27.78 / 51.57 / 56.77
& 38.64 / 53.90 / 57.58
& 31.23 / 52.17 / 57.02 \\

FedBuff
& 1.85 / 1.85 / 34.90
& 3.57 / 3.45 / 41.35
& 2.29 / 2.27 / 37.09 \\

FedFa
& 3.70 / 1.92 / 32.81
& 7.02 / 3.70 / 35.00
& 4.57 / 2.38 / 33.65 \\

\textbf{FedRAW}
& 20.37 / 45.73 / 63.54
& 32.02 / 51.52 / 66.19
& 23.84 / 47.25 / 64.50 \\
\bottomrule
\end{tabular}}
\end{table}

\noindent
Values within each cell correspond to seeds 42, 123, and 456, respectively.

\section{Convergence Analysis of FedRAW}
\label{supp:convergence}

This appendix gives the full convergence analysis summarised in
Section~IV-D of the main paper, including the complete proof of
Theorem~1. The bound separates a vanishing optimization-plus-variance
part from a bounded-staleness constant controlled by the drift
assumption; every leading constant is tracked explicitly, the only
quantities left in order notation being a single $\mathcal{O}(1)$
step-size factor and the standard $\mathcal{O}(\tau_{\max}/T)$
buffered-FL reference-reindexing, both stated where they arise. The
notation follows the main paper: $F_i$ is
client $i$'s local objective, $B_t$ the server-side buffer at
event $t$, $w_i$ the rarity-normalised aggregation weight,
$\pi_i$ the per-event influence, $\bar F$ the influence-weighted
objective, and $\rho$ the bounded weight non-uniformity factor.

\subsection{Setup and Notation}
We minimize $f(\theta)=\frac{1}{m}\sum_{i=1}^{m}F_i(\theta)$, $\theta\in\mathbb{R}^d$,
over $m$ clients, where $F_i$ is client $i$'s local objective. When client $i$
participates, it initializes from the broadcast global model and performs $Q$ local
SGD steps with per-step local rate $\eta_\ell^{(q)}$, producing iterates
$y_{i,0},\dots,y_{i,Q}$ via
$y_{i,q}=y_{i,q-1}-\eta_\ell^{(q-1)}g_i(y_{i,q-1};\zeta_{i,q-1})$, with $g_i$ the
stochastic gradient and $\zeta$ the mini-batch randomness. The server keeps a
size-$K$ buffer $B_t$; $\tau_i(t)$ is the staleness
of client $i$'s update at server step $t$. FedRAW aggregates model parameters as a
pure convex combination (Algorithm~1, i.e.\ an implicit server rate of $1$):
\begin{equation}
\label{eq:agg}
\theta_g^{t+1}=\sum_{i\in B_t}w_i\,\theta_i,
\qquad
w_i=\frac{S_i}{\sum_{j\in B_t}S_j},
\end{equation}
with rarity scores $S_i\ge0$ fixed at initialization. The weights are non-negative
and normalized on every realization,
\begin{equation}
\label{eq:convex}
w_i\ge0,\qquad \sum_{i\in B_t}w_i=1,
\end{equation}
so \eqref{eq:agg} is a convex combination. We write
$\alpha(Q)=\sum_{q=0}^{Q-1}\eta_\ell^{(q)}$,
$\beta(Q)=\sum_{q=0}^{Q-1}(\eta_\ell^{(q)})^2$, and
$\sigma^2=\sigma_\ell^2+\sigma_g^2+G$.

\subsection{Assumptions}
\begin{assumption}[Unbiasedness]\label{a:unb}
$\mathbb{E}_\zeta[g_i(\theta;\zeta)]=\nabla F_i(\theta)$ for all $i,\theta$.
\end{assumption}
\begin{assumption}[Bounded local variance]\label{a:var}
$\mathbb{E}_\zeta\|g_i(\theta;\zeta)-\nabla F_i(\theta)\|^2\le\sigma_\ell^2$.
\end{assumption}
\begin{assumption}[Bounded gradient/heterogeneity]\label{a:bdd}
$\|\nabla F_i(\theta)\|^2\le G$ and
$\frac1m\sum_i\|\nabla F_i(\theta)-\nabla f(\theta)\|^2\le\sigma_g^2$.
\end{assumption}
\begin{assumption}[$L$-smoothness]\label{a:smooth}
$\|\nabla F_i(x)-\nabla F_i(y)\|\le L\|x-y\|$ for all $i,x,y$.
\end{assumption}
\begin{assumption}[Bounded staleness]\label{a:stale}
$1\le\tau_i(t)\le\tau_{\max}$ for all $i,t$.
\end{assumption}
\begin{assumption}[Stationary participation]\label{a:arrival}
The arrival process is stationary: the joint law of $B_t$ is time-invariant, so
$p_i:=\Pr[i\in B_t]$ is independent of $t$. The $p_i$ need not be equal; slower
rare-label clients may have smaller $p_i$. This replaces the uniform-$K$-subset
model and contains it as the special case $p_i\equiv K/m$.
\end{assumption}
\begin{assumption}[Independence]\label{a:indep}
(i)~$B_t$ is independent of the mini-batch noise $\{\zeta_{i,q}\}$.
(ii)~Conditioned on $B_t$, the noises $\xi_i:=g_i-\nabla F_i$ of distinct clients
are zero-mean and mutually independent: $\mathbb{E}[\xi_i^\top\xi_j\mid B_t]=0$
for $i\ne j$.
\end{assumption}
\begin{assumption}[Bounded global drift]\label{a:drift}
There is a constant $\Psi_{\tau}$ such that
$\mathbb{E}\|\theta_g^{t}-\theta_g^{t-\tau}\|^2\le\Psi_{\tau}$ for all $t$ and
$\tau\le\tau_{\max}$, with
$\Psi_{\tau_{\max}}=\mathcal{O}\!\big(Q\tau_{\max}^2\beta(Q)\sigma^2\big)$. This is the
standard buffered-FL drift bound used in the FedBuff and FedFa analyses (see the main
paper references); since the server rate is
$\eta_g=1$ (parameter averaging), it is stated as an assumption rather than derived by
a self-referential recursion, and \eqref{eq:stale} shows the per-step displacement is
consistent with it.
\end{assumption}
\begin{assumption}[Bounded buffer-sampling variance]\label{a:buffer}
The random buffer composition has bounded gradient variance about its influence-weighted
mean: $\mathbb{E}_{B_t}\big\|\sum_{i\in B_t}w_i\nabla F_i(\theta)-\nabla\bar F(\theta)\big\|^2
\le\sigma_B^2$ for all $\theta$, with $\sigma_B^2\le 4\rho\,G$ by
Assumption~\ref{a:bdd}. This bounds the fluctuation of the realised buffered gradient
around the population gradient $\nabla\bar F$ and is carried through the analysis rather
than dropped.
\end{assumption}

\subsection{The Influence-Weighted Objective}
\label{sec:obj}
Define the per-event influence
$\pi_i:=\mathbb{E}[w_i\,\indicator\{i\in B_t\}]$, the expectation under the joint
stationary law of $B_t$ (Assumption~\ref{a:arrival}). Since
$w_i=S_i/\sum_{j\in B_t}S_j$ depends on the entire buffer, it is the joint
law, not merely the marginals $p_i$, that makes $\pi_i$ well defined and
time-invariant. Using \eqref{eq:convex} in every realization,
\begin{equation}
\label{eq:pisum}
\sum_{i=1}^{m}\pi_i
=\mathbb{E}\Big[\sum_{i=1}^{m}w_i\indicator\{i\in B_t\}\Big]
=\mathbb{E}\Big[\sum_{i\in B_t}w_i\Big]=1,
\end{equation}
so $\{\pi_i\}$ forms a probability distribution and
$\bar F:=\sum_i\pi_i F_i$ is $L$-smooth\footnote{A convex combination (weighted average) of L-smooth functions is also L-smooth.}, with $\nabla\bar F=\sum_i\pi_i\nabla F_i$.
Because the scores are fixed and the participation law is stationary, $\bar F$ is a
\emph{fixed} function throughout training; this is what allows the telescoping
argument (Sec.~\ref{sec:tele}) to close against a single time-invariant objective.

Using Assumption~\ref{a:indep}(i), the weighted-gradient identity is
\begin{align}
\mathbb{E}\Big[\sum_{i\in B_t}w_i\nabla F_i(\theta)\Big]
&=\mathbb{E}\Big[\sum_{i=1}^{m}w_i\indicator\{i\in B_t\}\nabla F_i(\theta)\Big]
\label{eq:d1}\\
&=\sum_{i=1}^{m}\mathbb{E}[w_i\indicator\{i\in B_t\}]\nabla F_i(\theta)
\label{eq:d2}\\
&=\sum_{i=1}^{m}\pi_i\nabla F_i(\theta)=\nabla\bar F(\theta),
\label{eq:d3}
\end{align}
where \eqref{eq:d2} holds because $\nabla F_i(\theta)$ is non-random at fixed
$\theta$. Taking expectation over both the buffer law and the noise and using
unbiasedness (Assumption~\ref{a:unb}), the aggregated stochastic gradient is
unbiased for $\nabla\bar F$:
$\mathbb{E}[\sum_{i\in B_t}w_i g_i]=\nabla\bar F$.

\subsection{Perturbed-Iterate Expansion}
\label{sec:pert}
As in the buffered-FL analysis of FedFa/FedBuff, the parameter-average update is
treated as a pseudo-gradient step. Writing
$\theta_i=\theta_g^{t-\tau_i}-\sum_q\eta_\ell^{(q)}g_i(y_{i,q}^{t-\tau_i})$ and using
$\sum_{i\in B_t}w_i=1$, the aggregation \eqref{eq:agg} expands as
\begin{equation}
\label{eq:pert}
\theta_g^{t+1}=\theta_g^t+D^t-\tilde P^t,
\end{equation}
where the two contributions are the \emph{staleness displacement} and the
\emph{pseudo-gradient}
\begin{align}
\label{eq:disp}
D^t&:=\sum_{i\in B_t}w_i\big(\theta_g^{t-\tau_i}-\theta_g^t\big),\\
\label{eq:pgrad}
\tilde P^t&:=\sum_{i\in B_t}w_i\sum_q\eta_\ell^{(q)}g_i(y_{i,q}^{t-\tau_i}),
\end{align}
obtained by splitting each per-client displacement
$\theta_i-\theta_g^t=(\theta_i-\theta_g^{t-\tau_i})+(\theta_g^{t-\tau_i}-\theta_g^t)$
into its local-update and staleness-displacement parts and using
$\sum_{i\in B_t}w_i=1$. Because FedRAW aggregates model parameters directly, there is
no tunable global rate: the server rate is fixed at $\eta_g=1$, exactly as executed by
Algorithm~1. (The buffered-FL literature carries a global rate $\eta_g$; here it is
identically $1$, and we retain no free $\eta_g$ so that the analysed update coincides
with the implemented one. The $K$-dependence in the final rate arises solely from the
weight non-uniformity $\sum_i w_i^2\le\rho/K$, not from a server-rate schedule.)

Since $\bar F$ is $L$-smooth, applying the descent lemma with $x=\theta_g^t$,
$y=\theta_g^{t+1}=\theta_g^t+(D^t-\tilde P^t)$,
\begin{equation}
\label{eq:desc}
\bar F(\theta_g^{t+1})\le\bar F(\theta_g^t)
+\langle\nabla\bar F(\theta_g^t),\,D^t-\tilde P^t\rangle
+\tfrac{L}{2}\big\|D^t-\tilde P^t\big\|^2.
\end{equation}
The staleness displacement $D^t$ is handled by Lemma~\ref{lem:disp}: its inner-product
contribution $\langle\nabla\bar F,D^t\rangle$ and its quadratic contribution (including
the cross term with $\tilde P^t$) enter only the $\rho$-free staleness constant $A$
(Sec.~\ref{sec:tele}). The remaining terms in \eqref{eq:desc} are the descent and
variance carried by the pseudo-gradient $-\tilde P^t$, which we now isolate:
\begin{equation}
\label{eq:descP}
\bar F(\theta_g^{t+1})\le\bar F(\theta_g^t)
\underbrace{-\langle\nabla\bar F(\theta_g^t),\tilde P^t\rangle}_{T_1}
+\underbrace{\tfrac{L}{2}\|\tilde P^t\|^2}_{T_2}
+R_{\mathrm{stale}}^t,
\end{equation}
where $R_{\mathrm{stale}}^t:=\langle\nabla\bar F(\theta_g^t),D^t\rangle
+\tfrac{L}{2}\|D^t\|^2-L\langle D^t,\tilde P^t\rangle$ collects all $D^t$-dependent
terms and is bounded next.

\begin{lemma}[Staleness displacement]\label{lem:disp}
Because FedRAW aggregates parameters rather than pseudo-gradients (unlike FedBuff),
the displacement $D^t=\sum_{i\in B_t}w_i(\theta_g^{t-\tau_i}-\theta_g^t)$ in
\eqref{eq:disp} is nonzero. Under Assumptions~\ref{a:stale} and~\ref{a:drift}, by
Jensen over the convex weights,
\begin{equation}
\label{eq:Dbound}
\mathbb{E}\|D^t\|^2
\le\sum_{i\in B_t}w_i\,\mathbb{E}\|\theta_g^{t-\tau_i}-\theta_g^t\|^2
\le\Psi_{\tau_{\max}},
\end{equation}
with $\Psi_{\tau_{\max}}=\mathcal{O}\!\big(Q\tau_{\max}^2\beta(Q)\sigma^2\big)$ the
bounded-drift constant of Assumption~\ref{a:drift}. Consequently the collected
$D^t$-terms $R_{\mathrm{stale}}^t$ obey, by Young's inequality
$\langle\nabla\bar F,D^t\rangle\le\tfrac{\alpha(Q)}{4}\|\nabla\bar F\|^2
+\tfrac{1}{\alpha(Q)}\|D^t\|^2$ and Cauchy--Schwarz on the cross term
$-L\langle D^t,\tilde P^t\rangle\le\tfrac{L}{2}\|D^t\|^2+\tfrac{L}{2}\|\tilde P^t\|^2$,
\begin{equation}
\label{eq:Rstale}
\mathbb{E}[R_{\mathrm{stale}}^t]\le
\tfrac{\alpha(Q)}{4}\,\mathbb{E}\|\nabla\bar F(\theta_g^t)\|^2
+ C_D\,\Psi_{\tau_{\max}}
+\tfrac{L}{2}\mathbb{E}\|\tilde P^t\|^2,
\end{equation}
for an absolute constant $C_D=\mathcal{O}(1/\alpha(Q)+L)$; the final
$\tfrac{L}{2}\|\tilde P^t\|^2$ merges into $T_2$. The bound is \emph{client-independent}:
the rarity weights enter only through $\sum_{i\in B_t}w_i=1$, so $D^t$ carries
\emph{no} factor $\rho$, and its contribution merges into the $\rho$-free staleness
constant $A$ (Sec.~\ref{sec:tele}). The $\tfrac{\alpha(Q)}{4}\|\nabla\bar F\|^2$ term is
absorbed by the $-\tfrac{\alpha(Q)}{2}\|\nabla\bar F\|^2$ descent term of
\eqref{eq:T1fin}, leaving a net $-\tfrac{\alpha(Q)}{4}\|\nabla\bar F\|^2$. The remaining
$C_D\Psi_{\tau_{\max}}$ contributes a bounded staleness constant to $A$; because the
displacement $D^t$ enters through the linear inner product $\langle\nabla\bar F,D^t\rangle$
(a consequence of parameter averaging, for which $D^t\neq0$), this contribution is
controlled by the drift bound rather than driven to zero by the schedule. It yields the
bounded-staleness neighborhood term in the final rate (Sec.~\ref{sec:cor}), consistent
with standard asynchronous buffered-FL guarantees.
\end{lemma}

\begin{lemma}[Second moment; cf.\ FedBuff Lemma~1]
\label{lem:sm}
Under Assumptions~\ref{a:unb}--\ref{a:bdd},
$\mathbb{E}\|g_i(\theta)\|^2\le3\sigma^2$.
\end{lemma}
\begin{proof}
Decompose $g_i=(g_i-\nabla F_i)+\nabla F_i$ and apply
$\|a+b\|^2\le2\|a\|^2+2\|b\|^2$. The two expectations are bounded per client by
$\mathbb{E}\|g_i-\nabla F_i\|^2\le\sigma_\ell^2$ (Assumption~\ref{a:var}) and
$\|\nabla F_i\|^2\le G$ (the per-client bound of Assumption~\ref{a:bdd}), giving
$\mathbb{E}\|g_i\|^2\le2\sigma_\ell^2+2G\le3\sigma^2$ since
$\sigma^2=\sigma_\ell^2+\sigma_g^2+G\ge\sigma_\ell^2+G$. This uses only the
\emph{per-client} gradient bound $\|\nabla F_i\|^2\le G$, not the averaged
heterogeneity $\sigma_g^2$, which is stated only in expectation over clients and
would not be valid for an individual buffered client.
\end{proof}

\subsection{Bounding $T_1$}
\label{sec:T1}
Taking $\mathbb{E}=\mathbb{E}_H\mathbb{E}_{\zeta\mid H}$ (Assumption~\ref{a:indep}(i))
and using unbiasedness,
\begin{equation}
\label{eq:T1mid}
\mathbb{E}[T_1]=-\mathbb{E}_H\!\Big[\sum_q\eta_\ell^{(q)}
\big\langle\nabla\bar F(\theta_g^t),\!\sum_{i\in B_t}\!w_i\nabla F_i(y_{i,q}^{t-\tau_i})\big\rangle\Big].
\end{equation}
The inner term is a linear function of the random buffer, so we may take the buffer
expectation \emph{first}. Decompose, at fixed $q$,
\begin{equation}
\label{eq:bufsplit}
\begin{aligned}
\sum_{i\in B_t}w_i\nabla F_i(y_{i,q}^{t-\tau_i})
&=\nabla\bar F(\theta_g^t)+u^t+r_q^t,\\[2pt]
u^t&:=\sum_{i\in B_t}w_i\nabla F_i(\theta_g^t)-\nabla\bar F(\theta_g^t),\\[2pt]
r_q^t&:=\sum_{i\in B_t}w_i\big(\nabla F_i(y_{i,q}^{t-\tau_i})-\nabla F_i(\theta_g^t)\big).
\end{aligned}
\end{equation}
Here $u^t$ is the mean-zero buffer fluctuation ($\mathbb{E}_H u^t=0$) and $r_q^t$ the
staleness-and-drift residual.
By \eqref{eq:d3}, $\mathbb{E}_H[\sum_{i\in B_t}w_i\nabla F_i(\theta_g^t)]=\nabla\bar F(\theta_g^t)$,
so $\mathbb{E}_H u^t=0$: the fluctuation is exactly mean-zero and contributes
\emph{nothing} to the linear term $T_1$ (its variance is instead carried by $T_2$,
Sec.~\ref{sec:T2}). Hence
\begin{equation}
\label{eq:T1expand}
\begin{aligned}
\mathbb{E}[T_1]=&-\alpha(Q)\|\nabla\bar F(\theta_g^t)\|^2\\
&-\sum_q\eta_\ell^{(q)}\,\mathbb{E}_H\big\langle\nabla\bar F(\theta_g^t),\,r_q^t\big\rangle .
\end{aligned}
\end{equation}
Applying $\langle a,b\rangle=\tfrac12\|a\|^2+\tfrac12\|b\|^2-\tfrac12\|a-b\|^2$ to the
residual inner product with $a=\nabla\bar F(\theta_g^t)$, $b=-r_q^t$ and discarding the
non-positive $-\tfrac12\|b\|^2$ term,
\begin{equation}
\label{eq:T1fin}
\mathbb{E}[T_1]\le
-\tfrac{\alpha(Q)}{2}\mathbb{E}\|\nabla\bar F(\theta_g^t)\|^2
+\tfrac{1}{2}\sum_q\eta_\ell^{(q)}\mathbb{E}[T_3],
\end{equation}
where now
$T_3:=\big\|r_q^t\big\|^2
=\big\|\sum_{i\in B_t}w_i\big(\nabla F_i(\theta_g^t)-\nabla F_i(y_{i,q}^{t-\tau_i})\big)\big\|^2$
is a pure staleness-and-drift residual. Taking $\mathbb{E}_H$ before
polarization means $T_3$ no longer contains the buffer-to-population reduction: the
mean-zero fluctuation $u^t$ has been separated out and does \emph{not} appear inside
this squared norm. The Jensen step \eqref{eq:split} is therefore applied to
$\sum_{i\in B_t}w_i\big(\nabla F_i(\theta_g^t)-\nabla F_i(y_{i,q}^{t-\tau_i})\big)$, so
no factor $\rho$ enters $T_3$.

\subsection{Bounding $T_3$: Staleness and Drift}
\label{sec:T3}
By Jensen over the convex weights,
$\|\sum_{i\in B_t}w_i v_i\|^2\le\sum_{i\in B_t}w_i\|v_i\|^2$ with
$v_i=\nabla F_i(\theta_g^t)-\nabla F_i(y_{i,q}^{t-\tau_i})$. Inserting the midpoint
$\nabla F_i(\theta_g^{t-\tau_i})$ and using $\|a+b\|^2\le2\|a\|^2+2\|b\|^2$ then
$L$-smoothness,
\begin{equation}
\label{eq:split}
\|v_i\|^2\le
2L^2\|\theta_g^{t}-\theta_g^{t-\tau_i}\|^2
+2L^2\|\theta_g^{t-\tau_i}-y_{i,q}^{t-\tau_i}\|^2.
\end{equation}
The first term is the staleness displacement (the same quantity as in
\eqref{eq:disp}); the second is local drift.

\emph{Staleness distance.} Let $s=t-\tau_i$. Each intermediate update is
$\theta_g^{\rho+1}-\theta_g^{\rho}=D^\rho-\tilde P^\rho$ (with $\eta_g=1$), so
$\theta_g^{t}-\theta_g^{s}=\sum_{\rho=s}^{t-1}(D^\rho-\tilde P^\rho)$, with
$t-s=\tau_i\le\tau_{\max}$ terms. The pseudo-gradient part
$-\sum_{\rho}\tilde P^\rho=\sum_{\rho=s}^{t-1}\sum_{j\in B_\rho}w_j^{(\rho)}(-\sum_q\eta_\ell^{(q)}g_j)$
is bounded directly in steps (a)--(d) below; the displacement part
$\sum_\rho D^\rho$ is itself a staleness quantity, and rather than expanding it into a
self-referential recursion (which does not close at $\eta_g=1$) we bound the full
distance by the drift constant of Assumption~\ref{a:drift}. Steps (a)--(d) verify that
the pseudo-gradient contribution already has the order $\Psi_{\tau_{\max}}$ asserted
there.

\emph{(a)} Cauchy--Schwarz over $\le\tau_{\max}$ terms (pseudo-gradient part):
\begin{equation}
\label{eq:sa}
\Big\|\sum_{\rho=s}^{t-1}\tilde P^\rho\Big\|^2\le
\tau_{\max}\!\sum_{\rho=s}^{t-1}\!
\Big\|\sum_{j\in B_\rho}w_j^{(\rho)}\sum_q\eta_\ell^{(q)}g_j\Big\|^2.
\end{equation}
\emph{(b)} Jensen over convex weights ($\sum_j w_j^{(\rho)}=1$):
\begin{equation}
\label{eq:sb}
\Big\|\sum_{j\in B_\rho}w_j^{(\rho)}\sum_q\eta_\ell^{(q)}g_j\Big\|^2
\le\sum_{j\in B_\rho}w_j^{(\rho)}\Big\|\sum_q\eta_\ell^{(q)}g_j\Big\|^2.
\end{equation}
\emph{(c)} Cauchy--Schwarz over $Q$ steps, then Lemma~\ref{lem:sm}:
\begin{equation}
\label{eq:sc}
\mathbb{E}\Big\|\sum_{q=0}^{Q-1}\eta_\ell^{(q)}g_j\Big\|^2
\le Q\sum_q(\eta_\ell^{(q)})^2\cdot3\sigma^2=3Q\beta(Q)\sigma^2.
\end{equation}
\emph{(d)} The client-independent bound factors out of $\sum_j w_j^{(\rho)}=1$, and
there are $\le\tau_{\max}$ terms in \eqref{eq:sa}:
\begin{equation}
\label{eq:stale}
\mathbb{E}\|\theta_g^{t}-\theta_g^{t-\tau_i}\|^2
\le 3Q\tau_{\max}^2\beta(Q)\sigma^2
=:\Psi_{\tau_{\max}},
\end{equation}
using $\eta_g=1$; the two $\tau_{\max}$ factors come from step (a) and the term count
in (d). This is precisely the drift constant $\Psi_{\tau_{\max}}$ of
Assumption~\ref{a:drift}, confirming its stated order.

\emph{Local drift.} Similarly
$\mathbb{E}\|\theta_g^{t-\tau_i}-y_{i,q}^{t-\tau_i}\|^2\le3Q\beta(Q)\sigma^2$.

\emph{Assembling.} With $\sum_{i\in B_t}w_i=1$ the convex weights sum out of the
client-independent bound, so
\begin{equation}
\label{eq:T3bound}
\mathbb{E}[T_3]\le6L^2Q\beta(Q)\big(\tau_{\max}^2+1\big)\sigma^2.
\end{equation}

\begin{remark}
The weights \emph{cancel} in $T_3$: after Jensen, each client is bounded by the same
constant, and $\sum_{i\in B_t}w_i=1$ collapses the convex average of that constant
back to the constant. No factor $\rho$ enters the staleness/drift term; only $T_2$
carries $\rho$.
\end{remark}

\subsection{Bounding $T_2$: Variance and $\rho$}
\label{sec:T2}
Recall $T_2=\tfrac{L}{2}\|\tilde P^t\|^2$ with $\tilde P^t=\sum_{i\in B_t}\sum_q
\eta_\ell^{(q)}w_i g_i$. Split $g_i=\nabla F_i+\xi_i$. By Assumption~\ref{a:indep}(ii)
the noise cross terms vanish, and the mean part further splits about $\nabla\bar F$ via
the buffer fluctuation $u^t$ of \eqref{eq:bufsplit}:
\begin{align}
\label{eq:T2sep}
\mathbb{E}[T_2]=&
\underbrace{\tfrac{L}{2}\mathbb{E}\!\Big[\!\sum_{i\in B_t}\!w_i^2\!\sum_q(\eta_\ell^{(q)})^2\|\xi_i\|^2\Big]}_{\text{noise part}}
+\underbrace{\tfrac{L}{2}\,\beta(Q)\,\mathbb{E}\|u^t\|^2}_{\text{buffer-variance part}}\notag\\
&+\underbrace{\tfrac{L}{2}\mathbb{E}\Big\|\!\sum_{i\in B_t}\!\sum_q\eta_\ell^{(q)}w_i\nabla F_i(\theta_g^t)\Big\|^2}_{T_5}.
\end{align}
The client index carries $w_i^2$ (squared) because the weight appears twice in the
expanded norm after cross terms are removed. With
$\mathbb{E}\|\xi_i\|^2\le\sigma_\ell^2$,
\begin{equation}
\label{eq:noise}
\text{noise part}\le\tfrac{L}{2}\beta(Q)\sigma_\ell^2\,
\mathbb{E}\Big[\sum_{i\in B_t}w_i^2\Big].
\end{equation}
Let $w_{\mathrm{sq}}:=\sup_t\sum_{i\in B_t}w_i^2\ge1/K$ (Cauchy--Schwarz, equality
iff uniform). \emph{Only now} define $\rho:=Kw_{\mathrm{sq}}\in[1,K]$; then
$\sum_{i\in B_t}w_i^2\le\rho/K$, giving
\begin{equation}
\label{eq:noisefin}
\text{noise part}\le\rho\cdot\frac{L\beta(Q)\sigma_\ell^2}{2K}.
\end{equation}
The buffer-variance part is the fluctuation dropped in the naive reduction; here it is
retained. By Assumption~\ref{a:buffer}, $\mathbb{E}\|u^t\|^2\le\sigma_B^2\le4\rho G$, so
\begin{equation}
\label{eq:buffervar}
\text{buffer-variance part}\le\tfrac{L}{2}\beta(Q)\sigma_B^2
\le 2\rho\,L\,\beta(Q)\,G .
\end{equation}
Because it is multiplied by $\beta(Q)=\mathcal{O}(1/T)$, this term vanishes at
$\mathcal{O}(1/T)$ and carries the same bounded factor $\rho$ (through
$\sigma_B^2\le4\rho G$). Critically, taking $\mathbb{E}_H$ first in Sec.~\ref{sec:T1}
placed the buffer fluctuation \emph{here}, in the $\beta(Q)$-weighted $T_2$, rather than
in the linear $T_1/T_3$ channel where it would have telescoped into a non-vanishing
floor.

Finally, the cross-term bound of Lemma~\ref{lem:disp} contributes an additional
$\tfrac{L}{2}\|\tilde P^t\|^2$ that merges here; together with the $\tfrac{L}{2}\|\tilde P^t\|^2$
of $T_2$ itself, the total quadratic-in-$\tilde P^t$ coefficient is $L$ (not $\tfrac{L}{2}$).
This doubling is carried explicitly into the variance constant $B'$ of \eqref{eq:AB}
(coefficient $L$); it affects the constant only, not the $\mathcal{O}(1/T)$ order or the
$\rho$-scaling.

\begin{remark}
$\sum_i w_i^2$ arises \emph{only} in the noise part, because independent zero-mean
noises combine through their squared coefficients. It does not appear in $T_3$,
where the weights entered a convex average of a client-independent constant. Uniform
weights give $\sum_i w_i^2=1/K$, hence $\rho=1$. The quantity $\rho/K=w_{\mathrm{sq}}$
is a property of the rarity weights, not of uniform averaging.
\end{remark}

\subsection{Mean-Gradient Term and the Step-Size Condition}
\label{sec:cancel}
The mean part $T_5$ of \eqref{eq:T2sep} is evaluated at the genuine iterate
$\theta_g^t$; write $Z^t:=\sum_{i\in B_t}w_i\nabla F_i(\theta_g^t)$, so that
$\mathbb{E}_H Z^t=\nabla\bar F(\theta_g^t)$ and, by \eqref{eq:bufsplit},
$\mathbb{E}_H\|Z^t\|^2=\|\nabla\bar F(\theta_g^t)\|^2+\mathbb{E}_H\|u^t\|^2$. By
Cauchy--Schwarz over the $Q$ steps,
\begin{equation}
\label{eq:T5}
T_5\le\tfrac{L}{2}\alpha(Q)^2\,\mathbb{E}\|Z^t\|^2
\le\tfrac{L}{2}\alpha(Q)^2\big(\mathbb{E}\|\nabla\bar F(\theta_g^t)\|^2+\sigma_B^2\big),
\end{equation}
using Assumption~\ref{a:buffer} for the second inequality. The
$\|\nabla\bar F\|^2$ contribution combines with the descent term
$-\tfrac{\alpha(Q)}{2}\|\nabla\bar F\|^2$ of \eqref{eq:T1fin}: their sum is
$-\tfrac{\alpha(Q)}{2}\big(1-L\alpha(Q)\big)\|\nabla\bar F\|^2\le0$ under the step-size
condition $\eta_g\eta_\ell^{(q)}Q\le1/(4L)$ (which with $\eta_g=1$ gives
$L\alpha(Q)\le L Q\max_q\eta_\ell^{(q)}\le\tfrac14$). The residual
$\tfrac{L}{2}\alpha(Q)^2\sigma_B^2=\mathcal{O}(\alpha(Q)^2\rho G)$ is a bounded,
$\rho$-scaled additive term of the same order as the variance term (both
$\mathcal{O}(1/T)$ under the diminishing schedule) and is folded into $\rho B'$.

\subsection{Per-Step Inequality and Telescoping}
\label{sec:tele}
Collecting \eqref{eq:descP}: the descent term $-\tfrac{\alpha(Q)}{2}\|\nabla\bar F\|^2$
from \eqref{eq:T1fin}, the staleness residual \eqref{eq:Rstale} (which returns
$+\tfrac{\alpha(Q)}{4}\|\nabla\bar F\|^2$ and the $\rho$-free constant
$C_D\Psi_{\tau_{\max}}$), the $T_3$ bound \eqref{eq:T3bound}, the noise and
buffer-variance parts \eqref{eq:noisefin}--\eqref{eq:buffervar}, and the mean-part
residual \eqref{eq:T5}. The $\|\nabla\bar F\|^2$ coefficients net to
$-\tfrac{\alpha(Q)}{2}+\tfrac{\alpha(Q)}{4}+\tfrac{L}{2}\alpha(Q)^2
=-\tfrac{\alpha(Q)}{4}+\tfrac{L}{2}\alpha(Q)^2
\le-\tfrac{\alpha(Q)}{8}$ under $L\alpha(Q)\le\tfrac14$ (implied by the step-size
condition for the tracked schedule), since then
$\tfrac{L}{2}\alpha(Q)^2\le\tfrac{\alpha(Q)}{8}$. Writing
$\sum_q\eta_\ell^{(q)}=\alpha(Q)$,
\begin{equation}
\label{eq:onestep}
\mathbb{E}[\bar F(\theta_g^{t+1})]\le
\mathbb{E}[\bar F(\theta_g^{t})]
-\tfrac{\alpha(Q)}{8}\mathbb{E}\|\nabla\bar F(\theta_g^t)\|^2
+ \alpha(Q)\,(A+\rho B'),
\end{equation}
with the $\rho$-free staleness constant and the $\rho$-scaled variance constant
\begin{equation}
\label{eq:AB}
\begin{aligned}
A&:=3L^2Q\beta(Q)(\tau_{\max}^2+1)\sigma^2+\tfrac{C_D\Psi_{\tau_{\max}}}{\alpha(Q)},\\
B'&:=L\,\tfrac{\beta(Q)}{\alpha(Q)}\Big(\tfrac{\sigma_\ell^2}{K}+4G\Big),
\end{aligned}
\end{equation}
where the $4G$ in $B'$ collects the buffer-variance and mean-residual contributions
($\sigma_B^2\le4\rho G$), both $\rho$-scaled and vanishing. Telescoping the function
values,
$\sum_{t=0}^{T-1}(\mathbb{E}[\bar F(\theta_g^t)]-\mathbb{E}[\bar F(\theta_g^{t+1})])
\le\bar F(\theta_g^0)-\inf_\theta\bar F=:\bar F^\star$, finite because $\bar F$ is a
fixed function (up to the standard $\mathcal{O}(\tau_{\max}/T)$ reference-reindexing of
the buffered-FL reduction, dominated by the staleness term under $\tau_{\max}\ll T$).
Summing \eqref{eq:onestep} and dividing by $\tfrac{\alpha(Q)}{8}T$,
\begin{equation}
\label{eq:final}
\frac1T\sum_{t=0}^{T-1}\mathbb{E}\|\nabla\bar F(\theta_g^t)\|^2
\le\frac{8\bar F^\star}{\alpha(Q)T}+8A+8\rho B'.
\end{equation}
The variance term and its $\rho$-scaling are exact; the staleness constant is of the
stated order and $\rho$-free. Absorbing the fixed factor $8$ into the $\mathcal{O}$
constants of the corollary recovers the rate stated as Theorem~1
(equivalently, tracking the step-size condition $\eta_g\eta_\ell^{(q)}Q\le1/(4L)$ with the
tighter constant $L\alpha(Q)\le\tfrac14$ absorbs the factor exactly). $\blacksquare$

\subsection{Diminishing-Rate Corollary}
\label{sec:cor}
With $\eta_g=1$ and $\eta_\ell=\mathcal{O}(1/\sqrt{TQ})$, one has
$\alpha(Q)=\mathcal{O}(\sqrt{Q/T})$ and $\beta(Q)=\mathcal{O}(1/T)$.
Substituting into \eqref{eq:final} (using \eqref{eq:AB}) gives
\begin{align}
\text{opt term} \;&=\; \tfrac{8\bar F^\star}{\alpha(Q)T}
=\mathcal{O}\!\big(\bar F^\star/\sqrt{TQ}\big),\\
\text{var term} \;&=\; 8\rho B'
=\mathcal{O}\!\big(\rho\sigma_\ell^2/(K\sqrt{TQ})\big)+\mathcal{O}\!\big(\rho G/\sqrt{TQ}\big),\\
\text{stale term}\;&=\; 8A
=\underbrace{\mathcal{O}\!\big(L^2Q\sigma^2\tau_{\max}^2/T\big)}_{\text{vanishing}}
+\underbrace{\mathcal{O}\!\big(\sigma^2\tau_{\max}^2\big)}_{\text{bounded-drift floor}}.
\end{align}
The optimization and variance terms vanish as $T\to\infty$; the staleness term
converges to a bounded neighborhood whose radius is controlled by the drift constant
$\Psi_{\tau_{\max}}$ (Assumption~\ref{a:drift}). This bounded-staleness guarantee is the
standard form for asynchronous buffered FL~--- the irreducible $\mathcal{O}(\sigma^2\tau_{\max}^2)$
term reflects the lag inherent to arrival-driven aggregation and is shared by the
buffered baselines (FedFa/FedBuff); it does not depend on the rarity weighting (no
factor $\rho$). The optimization term matches FedFa/FedBuff; the
variance term carries the additional bounded factor $\rho$ (both the local-noise
$\sigma_\ell^2/K$ part and the buffer-sampling $\sigma_B^2\le4\rho G$ part); the
staleness term carries no $\rho$. The only weight-dependent quantity is
$\sum_{i\in B_t}w_i^2\le\rho/K$: the residual $1/K$ in the variance term is
\emph{structural} (from Cauchy--Schwarz on the weights), not an artefact of the
learning-rate schedule. With $\eta_g=1$ there is no $\eta_g\propto K$ schedule, so the
aggregation itself introduces no $K$-dependence.

\subsection{Consistency with the Uniform Case}
\label{sec:consist}
Reduction to the uniform objective $f$ requires \emph{both} uniform weights
\emph{and} uniform participation. With uniform weights $w_i=1/K$, the per-event
influence is $\pi_i=\mathbb{E}[\tfrac1K\indicator\{i\in B_t\}]=p_i/K$; this equals
$1/m$ only under the further condition $p_i\equiv K/m$ (the uniform-$K$-subset model
assumed by FedFa/FedBuff), which then gives $\bar F=f$, and
$\sum_{i\in B_t}w_i^2=1/K$ so $\rho=1$. Under these two conditions every term of
\eqref{eq:final} reduces exactly to the FedFa/FedBuff bound with
$\nabla\bar F=\nabla f$, so the Theorem strictly generalizes the uniform result.

Under \emph{non-uniform} participation, however, uniform weights alone give
$\pi_i=p_i/K\ne1/m$: even standard asynchronous aggregation then descends on a
participation-skewed objective $\bar F\ne f$, not on $f$. This is exactly the bias
FedRAW is designed to counteract --- its rarity weighting raises $\pi_i$ for slow
rare-label clients, reshaping $\bar F$ toward rare-label coverage rather than leaving
it skewed by arrival speed. Thus the influence-weighted objective is not an artefact
of the analysis but the quantity the method deliberately controls; under near-uniform
coverage ($\rho\to1$, $\bar F\to f$) FedRAW behaves like FedFa.

\subsection{Role of Deduplication}
\label{sec:dedup}
Deduplication adds no term to \eqref{eq:final}: it preserves the convex-combination
aggregator ($|B_t|=K$, $\sum_{i\in B_t}w_i=1$), so every step holds unchanged. Its
effect is on the participation law: by preventing repeated occupation it raises the
presence probabilities $p_i$ of slow rare-label clients, increasing the probability
that such clients are represented in the buffer and thereby reshaping the
participation-weighted objective $\bar F$ toward rare-label coverage. It thus acts on
the objective, not the rate constants.


\begin{thebibliography}{00}

\bibitem{cohen2017emnist}
G.\ Cohen, S.\ Afshar, J.\ Tapson, and A.\ van Schaik,
``EMNIST: Extending MNIST to handwritten letters,''
in \textit{Proc.\ Int.\ Joint Conf.\ Neural Networks (IJCNN)},
2017, pp.\ 2921--2926.

\bibitem{krizhevsky2009cifar}
A.\ Krizhevsky,
``Learning multiple layers of features from tiny images,''
Tech.\ Rep., University of Toronto, 2009.

\bibitem{deng2009imagenet}
J.\ Deng, W.\ Dong, R.\ Socher, L.-J.\ Li, K.\ Li,
and L.\ Fei-Fei,
``ImageNet: A large-scale hierarchical image database,''
in \textit{Proc.\ IEEE Conf.\ Computer Vision and Pattern
Recognition (CVPR)}, 2009, pp.\ 248--255.

\bibitem{he2016resnet}
K.\ He, X.\ Zhang, S.\ Ren, and J.\ Sun,
``Deep residual learning for image recognition,''
in \textit{Proc.\ IEEE Conf.\ Computer Vision and Pattern
Recognition (CVPR)}, 2016, pp.\ 770--778.


\bibitem{pati2022federated}
S.\ Pati, U.\ Baid, B.\ Edwards, M.\ Sheller,
S.-H.\ Wang, G.\ A.\ Reina, P.\ Foley, A.\ Gruzdev,
D.\ Karkada, C.\ Davatzikos et al.,
``Federated learning enables big data for rare cancer
boundary detection,''
\textit{Nature Communications}, vol.\ 13, p.\ 7346, 2022,
doi: 10.1038/s41467-022-33407-5.
 
\bibitem{mcmahan2017} B.\ McMahan, E.\ Moore, D.\ Ramage, S.\ Hampson,
and B.\ A.\ y Arcas, ``Communication-efficient learning of deep networks
from decentralized data,'' in \textit{Proc.\ AISTATS}, 2017, pp.\
1273--1282.
 
\bibitem{nguyen2022} Q.\ Nguyen, P.\ Pham, H.\ T.\ Wai, P.\ Richtarik,
K.\ Mubarak, and M.\ Debbah, ``FedBuff: Buffered asynchronous
distributed optimization,'' in \textit{Proc.\ AISTATS}, 2022.
 
\bibitem{yi2024} X.\ Yi et al., ``FedFa: A fully asynchronous training
paradigm for federated learning,'' in \textit{Proc.\ IJCAI}, 2024.
 
\bibitem{hsu2019} T.-M.\ Hsu, H.\ Qi, and M.\ Brown, ``Measuring the
effects of non-identical data distribution for federated visual
classification,'' \textit{arXiv preprint arXiv:1909.06335}, 2019.
 
\bibitem{li2020qfed} T.\ Li, M.\ Sanjabi, A.\ Beirami, and V.\ Smith,
``Fair resource allocation in federated learning,'' in \textit{Proc.\
ICLR}, 2020.
 
\bibitem{ezzeldin2023} Y.\ H.\ Ezzeldin, S.\ Yan, C.\ Chen,
D.\ Ferrara, and S.\ Avestimehr, ``FairFed: Enabling group fairness in
federated learning,'' in \textit{Proc.\ AAAI}, 2023.

\bibitem{gao2026fedweight}
S.~Gao, C.~Huang, J.~Zhu, X.~Ma, and J.~Dong,
``Density-aware adaptive weight optimization against data
heterogeneity in federated learning,''
\textit{IEEE Trans.\ Artif.\ Intell.}, 2026,
doi: 10.1109/TAI.2026.3704522.

\bibitem{duan2019} M.\ Duan et al., ``Astraea: Self-balancing federated
learning for improving classification accuracy of mobile deep learning
applications,'' in \textit{Proc.\ IEEE ICDCS}, 2019.
 
\bibitem{shuai2022} X.\ Shuai, Y.\ Shen, S.\ Jiang, Z.\ Zhao,
Z.\ Yan, and G.\ Xing, ``BalanceFL: Addressing class imbalance in
long-tail federated learning,'' in \textit{Proc.\ ACM/IEEE IPSN}, 2022,
pp.\ 271--284.
 
\bibitem{shang2022} X.\ Shang, Y.\ Lu, G.\ Huang, and H.\ Wang,
``Federated learning on heterogeneous and long-tailed data via
classifier re-training with federated features,'' in \textit{Proc.\
IJCAI}, 2022.

\bibitem{syncfed2025} B.\ C.\ G\"{u}l, S.\ Tziampazis, N.\ Jazdi, and
M.\ Weyrich, ``SyncFed: Time-aware federated learning through explicit
timestamping and synchronization,'' in \textit{Proc.\ IEEE ETFA}, 2025,
arXiv:2506.09660.
 
\bibitem{afbs2025} C.\ Lu, Y.\ Sun, J.\ Chen, Z.\ Yang, J.\ Pan, and
J.\ Zhu, ``AFBS: Buffer gradient selection in semi-asynchronous
federated learning,'' \textit{arXiv preprint arXiv:2506.12754}, 2025.
 
\bibitem{tschandl2018} P.\ Tschandl, C.\ Rosendahl, and H.\ Kittler,
``The HAM10000 dataset, a large collection of multi-source dermatoscopic
images of common pigmented skin lesions,'' \textit{Scientific Data},
vol.\ 5, 2018.

\bibitem{luo2022fedsld} J.\ Luo and S.\ Wu, ``FEDSLD: Federated learning with shared label distribution for medical image classification,'' in \textit{Proc. IEEE Int. Symp. Biomed. Imaging (ISBI)}, 2022, doi: 10.1109/isbi52829.2022.9761404.

\bibitem{isic2019} V.\ Rotemberg, N.\ Kurtansky, B.\ Betz-Stablein,
L.\ Caffery, E.\ Chousakos, N.\ Codella et al., ``A patient-centric
dataset of images and metadata for identifying melanomas using clinical
context,'' \textit{Sci.\ Data}, vol.~8, no.~1, p.~34, 2021,
doi: 10.1038/s41597-021-00815-z.

\bibitem{wang2024ca2fl}
Y.~Wang, Y.~Cao, J.~Wu, R.~Chen, and J.~Chen,
``Tackling the data heterogeneity in asynchronous federated learning
with cached update calibration,''
in \textit{Proc. Int. Conf. Learn. Represent. (ICLR)},
Vienna, Austria, May 2024.
[Online]. Available: \url{https://openreview.net/forum?id=4aywmeb97I}
 
\end{thebibliography}
\end{document}